\documentclass[letterpaper]{article} 
\usepackage[preprint]{aaai2027}  
\usepackage[hyphens]{url}  
\usepackage{graphicx} 
\usepackage{natbib}  
\usepackage{caption} 
\usepackage{algorithm}
\usepackage{algorithmic}
\usepackage{booktabs} 
\usepackage[table]{xcolor} 
\usepackage{multirow}
\usepackage{booktabs}
\usepackage{array}
\usepackage{amsmath}
\usepackage{amsfonts}
\usepackage{amssymb}
\usepackage{amsthm}
\usepackage{bm}
\usepackage{mathtools}
\usepackage[capitalise,noabbrev]{cleveref}
\usepackage{sectsty}
\usepackage{multirow}
\usepackage{caption}

\newtheorem{theorem}{Theorem}[section]
\newtheorem{proposition}[theorem]{Proposition}
\newtheorem{corollary}[theorem]{Corollary}

\theoremstyle{definition}

\newtheorem{example}[theorem]{Example}

\theoremstyle{remark}
\newtheorem{remark}[theorem]{Remark}

\newcommand{\R}{\mathbb{R}}

\newcommand{\norm}[1]{\left\lVert #1 \right\rVert}

\newcommand{\relu}[1]{\left[#1\right]_{+}}
\newcommand{\diag}{\operatorname{diag}}
\newcommand{\col}{\operatorname{col}}
\newcommand{\sign}{\operatorname{sign}}

\newcommand{\cond}{\kappa}
\newcommand{\TP}{\mathrm{TP}}
\newcommand{\Raw}{\mathrm{Raw}}
\newcommand{\CI}{\mathrm{CI}}
\newcommand{\CD}{\mathrm{CD}}

\newcounter{secondcontribfn}
\newcommand{\secondcontrib}{%
    \ifnum\value{secondcontribfn}=0%
        \footnote{These authors contributed equally.}%
        \setcounter{secondcontribfn}{\value{footnote}}%
    \else%
        \footnotemark[\value{secondcontribfn}]%
    \fi%
}

\allowdisplaybreaks

\definecolor{LightGray}{gray}{0.9}
\newcolumntype{g}{>{\columncolor{gray!15}}c}
\usepackage{newfloat}
\usepackage{listings}
\DeclareCaptionStyle{ruled}{labelfont=normalfont,labelsep=colon,strut=off} 
\floatstyle{ruled}
\newfloat{listing}{tb}{lst}{}
\floatname{listing}{Listing}

\title{CryptoL: Towards Scale Dominance and Physics Constraints Mitigation in Financial Multivariate Time Series Forecasting}
\author{
    Yalda Taheri\textsuperscript{\rm 1}\equalcontrib,
    Mohammad Hassan Heydari\textsuperscript{\rm 2}\equalcontrib,
    Armon Rasooli\textsuperscript{\rm 3}\secondcontrib,\\
    Maryam Amirshahkarami\textsuperscript{\rm 2}\secondcontrib,
    Mohammad Ebrahim Mahdavi\textsuperscript{\rm 2}\secondcontrib,
    Hossein Karshenas\textsuperscript{\rm 2}
}
\affiliations{
    \textsuperscript{\rm 1}Faculty of Engineering, Azad University\\
    \textsuperscript{\rm 2}Faculty of Computer Engineering, University of Isfahan\\
    \textsuperscript{\rm 3}Department of Electrical Engineering, Iran University of Science and Technology\\
    y.jaliltaheri@iau.ir, m.heydari@mehr.ui.ac.ir, a.rasouli@ec.iut.ac.ir\\
    maryamamirshahkarami@mehr.ui.ac.ir, mhm.ebrh.mahdavi@gmail.com, h.karshenas@eng.ui.ac.ir
}

\begin{document}

\maketitle

\begin{abstract}
Cryptocurrency forecasting presents a distinctive combination of extreme cross-asset scale heterogeneity, non-stationary dynamics, and structural dependencies among Open–High–Low–Close (OHLC) variables. We present CryptoL, a unified framework designed to address these challenges within multivariate time-series forecasting. CryptoL evaluates forecasting error in context-normalized coordinates within the RevIN pipeline, preventing inverse normalization from introducing an additional squared-scale weighting into the MSE objective. We formally characterize this effect through the empirical risk and parameter-gradient geometry, establishing the conditions under which large-scale assets can disproportionately influence shared-model optimization. Beyond loss-space normalization, CryptoL examines channel-independent and channel-dependent normalization for OHLC data, showing that a shared channel-dependent affine transformation preserves candle-order relations that independent channel transformations need not preserve. The framework further incorporates scale-adaptive numerical stabilization to reduce distortions caused by a fixed normalization constant across assets spanning many orders of magnitude, together with a soft feasibility loss that penalizes violations of the defining OHLC inequalities. Experiments across heterogeneous cryptocurrency assets evaluate these components through controlled ablations and demonstrate improvements in forecasting accuracy, training stability, and the frequency of financially valid OHLC predictions relative to the considered baselines. CryptoL therefore provides an integrated approach to scale-balanced optimization, structure-preserving normalization, numerical stabilization, and constraint-aware cryptocurrency forecasting.
\end{abstract}


\section{Introduction}

Multivariate financial forecasting using open–high–low–close (OHLC) data provides a detailed representation of price evolution, but predicting these highly coupled variables jointly is challenging due to abrupt regime changes and strict physical candlestick rules \cite{ohlcmultivariate, ohlcmultivariate2, ohlcbitcoin, uncons}. This task is especially demanding in cryptocurrency markets, which exhibit extreme non-stationarity and cross-asset scale heterogeneity \cite{ohlcbitcoin, revin, revinanalysis}. Models must learn simultaneously from assets ranging from fractions of a cent to tens of thousands of dollars, making joint optimization difficult when using standard forecasting backbones that struggle with temporal distribution shifts and scale differences \cite{revin, san, fan, revinanalysis, gtt}.

\begin{figure*}[t]
    \centering
    \includegraphics[width=\textwidth]{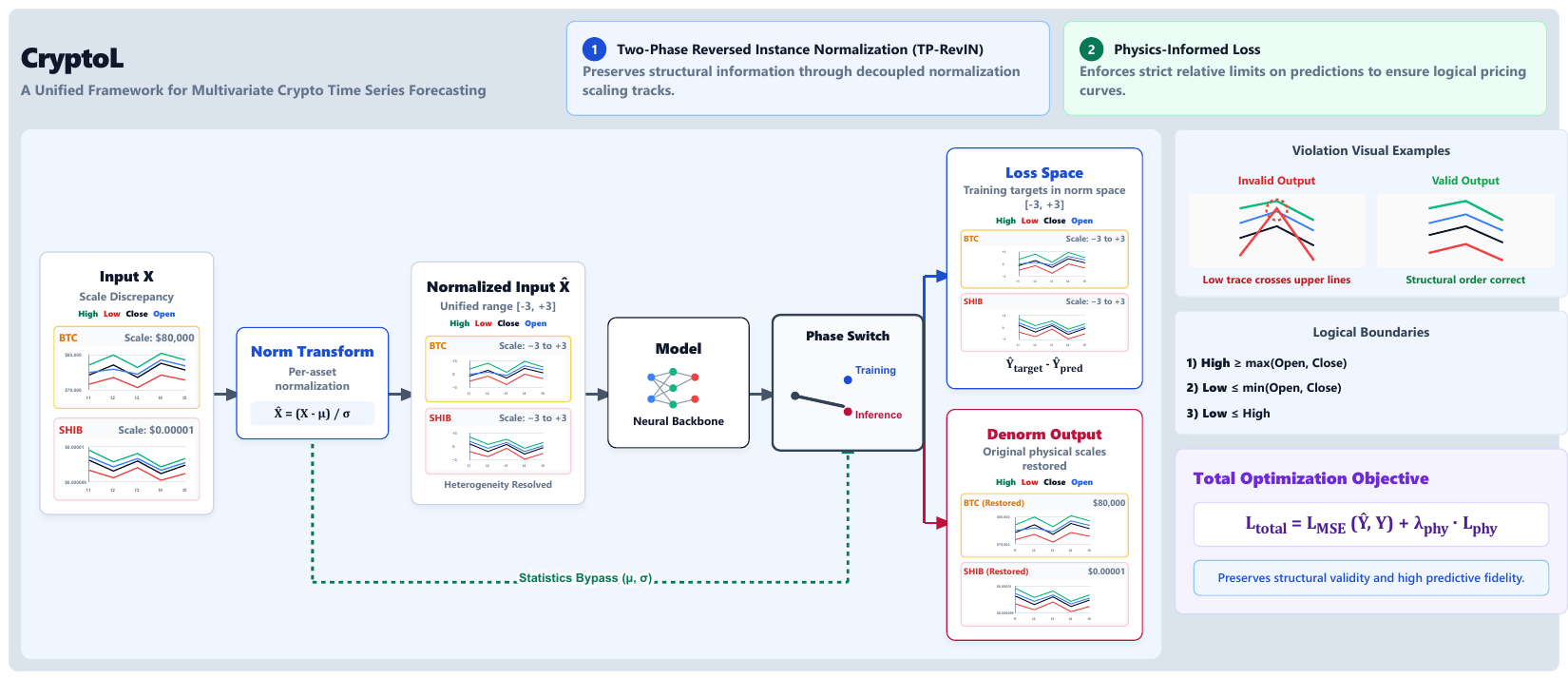}
    \caption{CryptoL framework}
    \label{fig:cryptol}
\end{figure*}

While Reversible Instance Normalization (RevIN) is commonly used to mitigate distribution shifts, its implementation details have significant structural and numerical consequences for OHLC data \cite{revin, ohlcbitcoin, oflcmachinelearning}. Channel-independent standardization can disrupt the relative ordering of candlestick channels, whereas channel-dependent normalization preserves these essential relationships. Additionally, standard RevIN evaluates forecasting errors in the original physical scale, which implicitly prioritizes high-value assets during shared model optimization \cite{revin, revinanalysis, san, fan, gtt}. Training directly in normalized target space which is referred to as Two-Phase RevIN (TP-RevIN), removes this scale-dependent weighting, while scale-adaptive stabilization prevents fixed normalization constants from distorting low-valued series \cite{revinanalysis, gtt}.

Even with proper normalization, models can still produce financially inadmissible predictions that violate core candlestick inequalities \cite{uncons, ohlcbitcoin}. Previous approaches addressed this structural consistency through specialized architectures or transformations, but these are often difficult to integrate with modern time-series foundation models. To address these limitations, we present \textsc{CryptoL}, a unified framework that combines structure-aware normalization, scale-balanced optimization via TP-RevIN, and an auxiliary physics-informed loss. This integrated approach discourages invalid candlestick structures while maintaining numerical stability and forecasting accuracy across diverse asset scales.

The contributions of this work are summarized as follows:
\begin{itemize}
    \item We systematically analyze channel-independent and channel-dependent RevIN for multivariate OHLC forecasting, demonstrate their different effects on candlestick structure, and investigate fixed and scale-adaptive epsilon formulations for stable normalization across assets with extreme numerical ranges.

    \item We provide a theoretical and extensive empirical analysis of TP-RevIN, equivalently normalized-space MSE training, showing how the location of the training objective affects scale-induced optimization imbalance in heterogeneous cryptocurrency forecasting.

    \item We introduce an auxiliary physics-informed loss that penalizes violations of the intrinsic OHLC ordering constraints and compare the resulting framework with unconstrained forecasting, conventional RevIN variants, and alternative normalization approaches.
\end{itemize}

\section{Related Work}

Multivariate financial forecasting aims to model temporal, cross-variable, and cross-asset dependencies jointly \cite{ohlcbitcoin, ohlcmultivariate, ohlcmultivariate2, oflcmachinelearning}. Classical approaches have used state-space models, Kalman filtering, and dynamic conditional correlation models, while recent methods employ recurrent networks, autoencoders, Transformers, and MLP-based architectures. StockMixer models indicator, temporal, and stock-level interactions, whereas foundation-model studies have shown that related multivariate inputs can improve financial forecasting \cite{stockmixer}. DPP instead learns representations directly from decomposed candlestick charts to predict future price movements \cite{dpp}.

OHLC forecasting introduces additional structural requirements because valid candlesticks must satisfy fixed relationships among the open, high, low, and close prices \cite{ohlcbitcoin, oflcmachinelearning, uncons}. Existing studies have addressed these constraints through invertible transformations that guarantee valid reconstructed outputs \cite{uncons}, as well as hybrid autoencoder and multitask architectures designed to model channel dependencies \cite{ohlcautoencoder}. Other work has incorporated the timestamps of OHLC events to enrich the bar representation \cite{ohlcmultivariate, ohlcmultivariate2, ohlcbitcoin, ohlc3}. These studies indicate that OHLC forecasting is not a standard multivariate regression problem, since predictions must preserve the internal structure of each candlestick.

Normalization methods have been widely studied for handling non-stationarity and distribution shift in time-series forecasting. RevIN removes instance-specific statistics before forecasting and restores them afterward \cite{revin}. SAN extends this principle through local temporal-slice normalization \cite{san}, while FAN uses dominant frequency components to address both trend and seasonal non-stationarity \cite{fan}. However, most normalization studies focus on temporal distribution changes and give less attention to how the normalization axis affects structured multivariate outputs such as OHLC channels.

Recent work has further distinguished representation normalization from the coordinate system in which the training loss is evaluated. GTT trains on targets normalized using input-context statistics, thereby learning directly in a unified curve-shape space \cite{gtt}. A comparative study of normalization in foundation models showed that MSE and MAE remain scale sensitive when predictions are denormalized before loss calculation \cite{comparative}. Another recent analysis of RevIN compared conventional and normalized backpropagation through normalized MSE and reported benefits for heterogeneous-scale data \cite{revinanalysis}. Building on these directions, our work jointly studies OHLC-aware normalization, normalized-space optimization, adaptive numerical stabilization, and auxiliary enforcement of candlestick validity.

\section{Methodology}

\subsection*{Preliminaries and Problem Definition}

Let
\begin{equation}
\begin{aligned}
\mathcal{D}
&=
\left\{
\left(\mathbf{X}_n,\mathbf{Y}_n\right)
\right\}_{n=1}^{N},\\
\mathbf{X}_n
&\in
\mathbb{R}^{L\times 4},
\qquad
\mathbf{Y}_n
\in
\mathbb{R}^{H\times 4},
\end{aligned}
\label{eq:forecasting_dataset}
\end{equation}
denote a collection of cryptocurrency forecasting samples, where
$\mathbf{X}_n$ is an observed context window of length $L$ and
$\mathbf{Y}_n$ is the corresponding forecast horizon of length $H$.
The four channels are ordered as
\begin{equation}
\mathbf{X}_{n,t,:}
=
\left(
O_{n,t},
H_{n,t},
L_{n,t},
C_{n,t}
\right),
\label{eq:ohlc_channel_order}
\end{equation}
representing the open, high, low, and close prices, respectively. A
valid OHLC observation belongs to the feasible set
\begin{equation}
\begin{aligned}
\Omega_{\mathrm{OHLC}}
=
\bigl\{
(O,H,L,C)\in\mathbb{R}^{4}:\;
&H\geq\max(O,C),\\
&L\leq\min(O,C)
\bigr\}.
\end{aligned}
\label{eq:ohlc_feasible_set}
\end{equation}
Given a forecasting model $f_{\boldsymbol{\theta}}$, the objective is
to estimate
\begin{equation}
\widehat{\mathbf{Y}}_n
=
f_{\boldsymbol{\theta}}(\mathbf{X}_n)
\label{eq:forecasting_objective}
\end{equation}
across cryptocurrencies whose numerical scales may differ by several
orders of magnitude, while maintaining stable optimization and
reducing violations of the constraints in
Equation~\eqref{eq:ohlc_feasible_set}. All normalization statistics are
computed exclusively from the observed context $\mathbf{X}_n$ and are
reused to normalize the associated target and restore the prediction
to its original physical scale.

\subsection{Channel-wise and Dynamic-Epsilon Normalization}
\label{subsec:ci_cd_dynamic_epsilon}

We investigate two choices for the axes over which RevIN statistics
are computed. In \emph{channel-independent} (CI) normalization, each
OHLC channel is normalized using its own context statistics. For
channel $c\in\{O,H,L,C\}$, these statistics are
\begin{equation}
\begin{aligned}
\mu_{n,c}^{\mathrm{CI}}
&=
\frac{1}{L}
\sum_{t=1}^{L}
X_{n,t,c},\\
v_{n,c}^{\mathrm{CI}}
&=
\frac{1}{L}
\sum_{t=1}^{L}
\left(
X_{n,t,c}
-
\mu_{n,c}^{\mathrm{CI}}
\right)^2,
\end{aligned}
\label{eq:ci_statistics}
\end{equation}
with effective scale
\begin{equation}
s_{n,c}^{\mathrm{CI}}
=
\sqrt{
v_{n,c}^{\mathrm{CI}}
+
\epsilon_{n,c}
}.
\label{eq:ci_scale}
\end{equation}
The context and target are then transformed as
\begin{equation}
\begin{aligned}
\widetilde{X}_{n,t,c}^{\mathrm{CI}}
&=
\frac{
X_{n,t,c}-\mu_{n,c}^{\mathrm{CI}}
}{
s_{n,c}^{\mathrm{CI}}
},\\
\widetilde{Y}_{n,\tau,c}^{\mathrm{CI}}
&=
\frac{
Y_{n,\tau,c}-\mu_{n,c}^{\mathrm{CI}}
}{
s_{n,c}^{\mathrm{CI}}
}.
\end{aligned}
\label{eq:ci_transform}
\end{equation}
CI normalization standardizes open, high, low, and close separately.
Because the four channels undergo different affine transformations,
their original ordering is not mathematically guaranteed to remain
unchanged in normalized space.

In \emph{channel-dependent} (CD) normalization, one shared mean and
variance are computed jointly over the temporal and channel
dimensions:
\begin{equation}
\begin{aligned}
\mu_n^{\mathrm{CD}}
&=
\frac{1}{4L}
\sum_{t=1}^{L}
\sum_{c=1}^{4}
X_{n,t,c},\\
v_n^{\mathrm{CD}}
&=
\frac{1}{4L}
\sum_{t=1}^{L}
\sum_{c=1}^{4}
\left(
X_{n,t,c}
-
\mu_n^{\mathrm{CD}}
\right)^2.
\end{aligned}
\label{eq:cd_statistics}
\end{equation}
The corresponding effective scale is
\begin{equation}
s_n^{\mathrm{CD}}
=
\sqrt{
v_n^{\mathrm{CD}}
+
\epsilon_n
},
\label{eq:cd_scale}
\end{equation}
and the same transformation is applied to all OHLC channels:
\begin{equation}
\begin{aligned}
\widetilde{X}_{n,t,c}^{\mathrm{CD}}
&=
\frac{
X_{n,t,c}-\mu_n^{\mathrm{CD}}
}{
s_n^{\mathrm{CD}}
},\\
\widetilde{Y}_{n,\tau,c}^{\mathrm{CD}}
&=
\frac{
Y_{n,\tau,c}-\mu_n^{\mathrm{CD}}
}{
s_n^{\mathrm{CD}}
}.
\end{aligned}
\label{eq:cd_transform}
\end{equation}
Unlike CI, CD applies one common affine transformation to the complete
OHLC sample and therefore retains the relative organization of the
candlestick channels. Further mathematical proofs and analysis of order
preservation under CD and CI normalization are provided in
Appendix~\ref{sec:mathematical_analysis}.

We additionally examine fixed and dynamic formulations of the
stabilizing term used in Equations~\eqref{eq:ci_scale}
and~\eqref{eq:cd_scale}. The fixed formulation is
\begin{equation}
\epsilon^{\mathrm{fix}}
=
10^{-5},
\label{eq:fixed_epsilon}
\end{equation}
whereas the dynamic formulation adapts the stabilizer to the magnitude
of the context mean:
\begin{equation}
\epsilon^{\mathrm{dyn}}
=
10^{-5}
\left(
\mu^2+10^{-12}
\right).
\label{eq:dynamic_epsilon}
\end{equation}
For CI normalization, $\mu$ denotes the corresponding channelwise
mean; for CD normalization, it denotes the shared sample-level mean.
Dynamic epsilon is designed to prevent a fixed numerical constant from
having a disproportionately large effect on cryptocurrencies with very
small quoted values. Additional mathematical analysis of the fixed and
dynamic epsilon formulations is provided in
Appendix~\ref{sec:mathematical_analysis}.

\begin{table*}[t]
\centering

\resizebox{\textwidth}{!}{%
\begin{tabular}{c c l | c c c | c c c | c c c}
\toprule

\multirow{2}{*}{\textbf{\shortstack{Time\\Frame}}} & \multirow{2}{*}{\textbf{Horizon}} & \multirow{2}{*}{\textbf{Norm Technique}} & 
\multicolumn{3}{c}{\textbf{Time-MoE}} & 
\multicolumn{3}{c}{\textbf{Timer-XL}} & 
\multicolumn{3}{c}{\textbf{Timer}} \\ 
\cmidrule(lr){4-6} \cmidrule(lr){7-9} \cmidrule(lr){10-12}

& & & 
\textbf{MSE} & \textbf{MAE} & \textbf{PHY} & 
\textbf{MSE} & \textbf{MAE} & \textbf{PHY} & 
\textbf{MSE} & \textbf{MAE} & \textbf{PHY} \\
\midrule

 & \multirow{5}{*}{5}  
 & w/o RevIN & 5.7e+8 & 6325.3 & 1.312 & 5.2e+8 & 6571.64 & 5.536 & 6.15e+8 & 6548.55 & 15.154 \\
 & & CI RevIN  & 9890.9 & 19.03 & 0.0040 & \textcolor{red}{4387.64} & \textcolor{red}{11.66} & 0.012 & 4722.96 & 12.37 & \textcolor{blue}{0.439} \\
 & & CD RevIN  & 9868.58 & 18.93 & 0.0076 & \textcolor{blue}{4485.7} & \textcolor{blue}{11.81} & 0.0050 & 5013.64 & 12.70 & 1.71 \\
 \rowcolor{LightGray} \cellcolor{white} & \cellcolor{white} & Two-Phase CI RevIN & \textcolor{red}{6610.0} & \textcolor{red}{15.25} & \textcolor{blue}{0.001} & 4584.75 & 12.34 & \textcolor{blue}{0.0005} & \textcolor{red}{4510.90} & \textcolor{red}{12.14} & \textcolor{red}{0.278} \\
 \rowcolor{LightGray} \cellcolor{white} & \cellcolor{white} & Two-Phase CD RevIN & \textcolor{blue}{7480.2} & \textcolor{blue}{16.31} & \textcolor{red}{0.00049} & 4811.0 & 12.72 & \textcolor{red}{0.0004} & \textcolor{blue}{4632.05} & \textcolor{blue}{12.30} & 1.09 \\
 \cmidrule{2-12}

 & \multirow{5}{*}{15} 
 & w/o RevIN & 5.7e+8 & 6328.50 & 1.827 & 5.5e+8 & 6572.12 & 5.532 & 6.15e+8 & 6548.99 & 9.523 \\
 & & CI RevIN  & 2.20e+4 & 28.63 & 0.0046 & 1.05e+4 & 18.54 & 0.0159 & 1.06e+4 & 18.69 & \textcolor{blue}{0.915} \\
 & & CD RevIN  & 2.23e+4 & 28.62 & \textcolor{blue}{0.0025} & 1.07e+4 & 18.71 & 0.0061 & 1.19e+4 & 19.73 & 6.285 \\
 \rowcolor{LightGray} \cellcolor{white} & \cellcolor{white} & Two-Phase CI RevIN & \textcolor{red}{1.53e+4} & \textcolor{red}{23.55} & 0.012 & \textcolor{blue}{1.02e+4} & \textcolor{blue}{18.39} & \textcolor{blue}{0.0003} & \textcolor{red}{9949.94} & \textcolor{red}{18.09} & \textcolor{red}{0.390} \\
 \rowcolor{LightGray} \cellcolor{white} \multirow{-10}{*}{\textbf{5m}} & \cellcolor{white} & Two-Phase CD RevIN & \textcolor{blue}{1.95e+4} & \textcolor{blue}{26.83} & \textcolor{red}{0.0010} & \textcolor{red}{1.01e+4} & \textcolor{red}{18.32} & \textcolor{red}{0.0002} & \textcolor{blue}{1.04e+4} & \textcolor{blue}{18.54} & 4.36 \\
\midrule\midrule

 & \multirow{5}{*}{5}  
 & w/o RevIN & 7.73e+8 & 8438.67 & 1.271 & 7.66e+8 & 8367.45 & 0.501 & 7.71e+8 & 8408.11 & 5.717 \\
 & & CI RevIN  & 6.47e+4 & 54.60 & 0.0685 & 3.35e+4 & 37.58 & 0.063 & 3.22e+4 & 36.13 & \textcolor{blue}{1.012} \\
 & & CD RevIN  & 6.33e+4 & 54.20 & 0.0307 & 3.38e+4 & 37.97 & 0.081 & 3.24e+4 & 36.39 & 3.957 \\
 \rowcolor{LightGray} \cellcolor{white} & \cellcolor{white} & Two-Phase CI RevIN & \textcolor{red}{4.67e+4} & \textcolor{red}{45.45} & \textcolor{blue}{0.011} & \textcolor{red}{3.11e+4} & \textcolor{red}{36.04} & \textcolor{blue}{0.003} & \textcolor{blue}{3.03e+4} & \textcolor{blue}{34.93} & \textcolor{red}{0.678} \\
 \rowcolor{LightGray} \cellcolor{white} & \cellcolor{white} & Two-Phase CD RevIN & \textcolor{blue}{5.41e+4} & \textcolor{blue}{49.12} & \textcolor{red}{0.0076} & \textcolor{blue}{3.11e+4} & \textcolor{blue}{36.26} & \textcolor{red}{0.00074} & \textcolor{red}{1.4e+4} & \textcolor{red}{18.54} & 4.365 \\
 \cmidrule{2-12}

 & \multirow{5}{*}{15} 
 & w/o RevIN & 7.73e+8 & 8442.13 & 1.553 & 7.68e+8 & 8383.54 & 0.606 & 7.71e+8 & 8412.84 & 5.542 \\
 & & CI RevIN  & 1.56e+5 & 86.77 & 0.0701 & 8.29e+4 & 59.11 & 0.228 & 7.66e+4 & 57.00 & \textcolor{blue}{3.238} \\
 & & CD RevIN  & 1.55e+5 & 87.05 & \textcolor{blue}{0.0443} & 8.49e+4 & 60.007 & 0.163 & 8.51e+4 & 59.75 & 20.67 \\
 \rowcolor{LightGray} \cellcolor{white} & \cellcolor{white} & Two-Phase CI RevIN & \textcolor{red}{1.14e+5} & \textcolor{red}{72.48} & 0.12 & \textcolor{blue}{7.34e+4} & \textcolor{blue}{55.10} & \textcolor{blue}{0.026} & \textcolor{red}{6.99e+4} & \textcolor{red}{53.38} & \textcolor{red}{1.123} \\
 \rowcolor{LightGray} \cellcolor{white} \multirow{-10}{*}{\textbf{30m}} & \cellcolor{white} & Two-Phase CD RevIN & \textcolor{blue}{1.37e+5} & \textcolor{blue}{80.14} & \textcolor{red}{0.0078} & \textcolor{red}{7.24e+4} & \textcolor{red}{54.71} & \textcolor{red}{0.00052} & \textcolor{blue}{7.36e+4} & \textcolor{blue}{55.11} & 15.002 \\

\bottomrule
\end{tabular}%
}

\captionsetup{justification=centering}
\caption{Comparison of Different RevIN techniques. Lower numbers indicate lower errors and better performances. All of the experiments in this table are done using training with MSE only. Best metrics in each category are colored with \textcolor{red}{Red} and Second best metrics are colored with \textcolor{blue}{Blue}.}
\label{tab:main_table}
\end{table*}

\begin{table*}[t]
\centering

\setlength{\aboverulesep}{0pt}
\setlength{\belowrulesep}{0pt}
\renewcommand{\arraystretch}{1.3} 

\begin{tabular}{ll ccc ccc ggg}
\toprule
\multirow{2}{*}{\textbf{Model}} & \multirow{2}{*}{\textbf{Horizon}} & \multicolumn{3}{c}{\textbf{FAN}} & \multicolumn{3}{c}{\textbf{SAN}} & \multicolumn{3}{>{\columncolor{gray!15}}c}{\textbf{TP-RevIN}} \\
\cmidrule(lr){3-5} \cmidrule(lr){6-8} \cmidrule(lr){9-11}
& & MAE & MAPE & PHY & MAE & MAPE & PHY & MAE & MAPE & PHY \\
\midrule
\multirow{3}{*}{Time-MoE} 
& 5  & 284.19 & 1.63e+4 & 181.12 & 34.99 & 1.45e+4 & 0.32 & 68.70 & 1.31 & 0.068 \\
& 15 & 285.92 & 1.92e+3 & 187.11 &  55.89 & 1.52e+3 & 0.15 & 108.05 & 2.07 & 0.91 \\
& 30 & 286.15 & 2.6e+4 & 183.43 & 83.23 & 5.21e+3 &  0.50 & 131.45 & 2.52 & 2.68 \\
\midrule
\multirow{3}{*}{Timer-XL} 
& 5  & 113.85 & 1.67e+4 & 12.68 & 36.03 & 2.40e+4 & 0.84 & 50.93 & 0.95 & 0.01 \\
& 15 & 123.66 & 8.0e+4 & 11.03 &  56.75 & 4.84e+3 & 1.46 & 80.36 & 1.49 & 0.016 \\
& 30 & 125.46 & 2.6e+5 & 11.04 & 81.05 & 9.58e+3 & 1.29 & 106.21 & 1.97 & 0.0097 \\
\bottomrule
\end{tabular}
\caption{Comparison of Models with FAN, SAN, and TP-RevIN Normalization Techniques. Experiments are done under 1h time frame. For TP-RevIN, here we use dynamic epsilon.}
\label{tab:san_fan_tp}
\end{table*}

\subsection{Two-Phase RevIN and Scale-Balanced Optimization}
\label{subsec:tp_revin}

Let the normalization statistics selected in the previous subsection be denoted by
$\boldsymbol{\mu}_n\in\mathbb{R}^{4}$ and
$\mathbf{S}_n\in\mathbb{R}^{4\times4}$, where
$\mathbf{S}_n$ is a positive diagonal channel-scale matrix. Under CD
normalization, $\mathbf{S}_n=s_n\mathbf{I}_4$; under CI normalization,
its diagonal entries contain the channelwise scales. After
normalizing the observed context and target, we write
\begin{equation}
\widetilde{\mathbf{X}}_n
=
\left(
\mathbf{X}_n-\mathbf{1}_L\boldsymbol{\mu}_n^{\top}
\right)
\mathbf{S}_n^{-1},
\qquad
\mathbf{Z}_n
=
\left(
\mathbf{Y}_n-\mathbf{1}_H\boldsymbol{\mu}_n^{\top}
\right)
\mathbf{S}_n^{-1},
\label{eq:tp_normalized_input_target}
\end{equation}
where the statistics are computed exclusively from
$\mathbf{X}_n$. The forecasting backbone produces a normalized
prediction
\begin{equation}
\widehat{\mathbf{Z}}_n
=
f_{\boldsymbol{\theta}}
\left(
\widetilde{\mathbf{X}}_n
\right),
\label{eq:tp_normalized_prediction}
\end{equation}
and the corresponding physical-scale prediction is recovered as
\begin{equation}
\widehat{\mathbf{Y}}_n
=
\mathbf{1}_H\boldsymbol{\mu}_n^{\top}
+
\widehat{\mathbf{Z}}_n\mathbf{S}_n.
\label{eq:tp_denormalization}
\end{equation}

Conventional RevIN-based training commonly restores the prediction
through Equation~\eqref{eq:tp_denormalization} before evaluating MSE.
We refer to this objective as original-space RevIN training:
\begin{equation}
\mathcal{L}_{\mathrm{R}}
\left(
\boldsymbol{\theta}
\right)
=
\frac{1}{2N}
\sum_{n=1}^{N}
\left\|
\widehat{\mathbf{Y}}_n-\mathbf{Y}_n
\right\|_F^2.
\label{eq:revin_raw_loss}
\end{equation}
In contrast, \emph{Two-Phase RevIN} (TP-RevIN) separates the
normalized training phase from the physical-scale inference phase.
During training, the objective is evaluated directly between the
normalized prediction and normalized target:
\begin{equation}
\mathcal{L}_{\mathrm{TP}}
\left(
\boldsymbol{\theta}
\right)
=
\frac{1}{2N}
\sum_{n=1}^{N}
\left\|
\widehat{\mathbf{Z}}_n-\mathbf{Z}_n
\right\|_F^2.
\label{eq:tp_revin_loss}
\end{equation}
Inverse normalization is therefore required for inference and
reporting, but it does not participate in the training objective.

\paragraph{Scale-induced reweighting in original-space RevIN.}
Let $\operatorname{rvec}(\cdot)$ stack the rows of an
$H\times4$ matrix, and define
\begin{equation}
\mathbf{D}_n
=
\mathbf{I}_H\otimes\mathbf{S}_n
\in
\mathbb{R}^{4H\times4H}.
\label{eq:vectorized_scale_matrix}
\end{equation}
Define the vectorized normalized residual and its parameter Jacobian as
\begin{equation}
\mathbf{e}_n
\left(
\boldsymbol{\theta}
\right)
=
\operatorname{rvec}
\left(
\widehat{\mathbf{Z}}_n-\mathbf{Z}_n
\right),
\qquad
\mathbf{J}_n
\left(
\boldsymbol{\theta}
\right)
=
\frac{
\partial\operatorname{rvec}
\left(
\widehat{\mathbf{Z}}_n
\right)
}{
\partial\boldsymbol{\theta}
}
\in\mathbb{R}^{4H\times p}.
\label{eq:normalized_residual_jacobian}
\end{equation}

\noindent\textbf{Proposition 1.}
Assume that $\boldsymbol{\mu}_n$ and $\mathbf{D}_n$ are computed from
the observed context and are independent of
$\boldsymbol{\theta}$. Then original-space RevIN MSE satisfies
\begin{equation}
\mathcal{L}_{\mathrm{R}}
\left(
\boldsymbol{\theta}
\right)
=
\frac{1}{2N}
\sum_{n=1}^{N}
\mathbf{e}_n^{\top}
\mathbf{D}_n^{2}
\mathbf{e}_n,
\label{eq:raw_loss_weighted_normalized}
\end{equation}
whereas TP-RevIN satisfies
Equation~\eqref{eq:tp_revin_loss}. Their parameter gradients are
\begin{equation}
\nabla_{\boldsymbol{\theta}}
\mathcal{L}_{\mathrm{R}}
=
\frac{1}{N}
\sum_{n=1}^{N}
\mathbf{J}_n^{\top}
\mathbf{D}_n^{2}
\mathbf{e}_n,
\label{eq:raw_revin_gradient}
\end{equation}
and
\begin{equation}
\nabla_{\boldsymbol{\theta}}
\mathcal{L}_{\mathrm{TP}}
=
\frac{1}{N}
\sum_{n=1}^{N}
\mathbf{J}_n^{\top}
\mathbf{e}_n.
\label{eq:tp_revin_gradient}
\end{equation}
Consequently, original-space RevIN introduces an explicit
scale-dependent weighting into both the empirical objective and the
shared parameter update, while TP-RevIN removes this weighting.

\paragraph{Proof.}
From Equations~\eqref{eq:tp_normalized_input_target}
and~\eqref{eq:tp_denormalization},
\begin{equation}
\begin{aligned}
\operatorname{rvec}
\left(
\widehat{\mathbf{Y}}_n-\mathbf{Y}_n
\right)
&=
\mathbf{D}_n
\operatorname{rvec}
\left(
\widehat{\mathbf{Z}}_n-\mathbf{Z}_n
\right)\\
&=
\mathbf{D}_n\mathbf{e}_n.
\end{aligned}
\label{eq:raw_normalized_residual_relation}
\end{equation}
Therefore,
\begin{equation}
\left\|
\widehat{\mathbf{Y}}_n-\mathbf{Y}_n
\right\|_F^2
=
\mathbf{e}_n^{\top}
\mathbf{D}_n^{\top}
\mathbf{D}_n
\mathbf{e}_n.
\label{eq:raw_residual_quadratic_form}
\end{equation}
Because $\mathbf{D}_n$ is diagonal and positive,
$\mathbf{D}_n^{\top}\mathbf{D}_n=\mathbf{D}_n^2$, which proves
Equation~\eqref{eq:raw_loss_weighted_normalized}.

Differentiating the per-sample term in
Equation~\eqref{eq:raw_loss_weighted_normalized} and treating
$\mathbf{D}_n$ as constant with respect to
$\boldsymbol{\theta}$ gives
\begin{equation}
\nabla_{\boldsymbol{\theta}}
\left(
\frac{1}{2}
\mathbf{e}_n^{\top}
\mathbf{D}_n^2
\mathbf{e}_n
\right)
=
\mathbf{J}_n^{\top}
\mathbf{D}_n^2
\mathbf{e}_n.
\label{eq:raw_sample_gradient}
\end{equation}
Summing over all samples yields
Equation~\eqref{eq:raw_revin_gradient}. Similarly,
differentiating the TP-RevIN term
$\frac{1}{2}\mathbf{e}_n^{\top}\mathbf{e}_n$ gives
$\mathbf{J}_n^{\top}\mathbf{e}_n$, establishing
Equation~\eqref{eq:tp_revin_gradient}.
\hfill$\square$

For CD normalization, where
$\mathbf{D}_n=s_n\mathbf{I}_{4H}$, Proposition~1 reduces to
\begin{equation}
\mathcal{L}_{\mathrm{R}}
=
\frac{1}{2N}
\sum_{n=1}^{N}
s_n^2
\left\|
\mathbf{e}_n
\right\|_2^2,
\label{eq:raw_scalar_scale_loss}
\end{equation}
and
\begin{equation}
\nabla_{\boldsymbol{\theta}}
\mathcal{L}_{\mathrm{R}}
=
\frac{1}{N}
\sum_{n=1}^{N}
s_n^2
\mathbf{J}_n^{\top}
\mathbf{e}_n.
\label{eq:raw_scalar_scale_gradient}
\end{equation}
Thus two samples with comparable normalized residuals and comparable
model sensitivity contribute to the original-space update in
proportion to the squares of their context scales. This is the precise
sense in which scale dominance arises in RevIN-based MSE training.
Large numerical scale alone does not determine the complete gradient,
since the contribution also depends on
$\mathbf{J}_n^{\top}\mathbf{e}_n$; however, the factor $s_n^2$ is an
unavoidable multiplicative component of the objective whenever the
loss is evaluated after inverse normalization.

The same result can be expressed at the asset level. Let
$\pi_a$ denote the sampling probability of asset $a$, and define its
normalized forecasting risk as
\begin{equation}
R_a
\left(
\boldsymbol{\theta}
\right)
=
\frac{1}{2}
\mathbb{E}
\left[
\left\|
\mathbf{e}
\left(
\boldsymbol{\theta}
\right)
\right\|_2^2
\,\middle|\,
a
\right].
\label{eq:asset_normalized_risk}
\end{equation}
If asset $a$ has approximately constant context scale $s_a$, then
original-space RevIN minimizes
\begin{equation}
R_{\mathrm{R}}
\left(
\boldsymbol{\theta}
\right)
=
\sum_{a}
\pi_a s_a^2
R_a
\left(
\boldsymbol{\theta}
\right),
\label{eq:raw_asset_risk}
\end{equation}
whereas TP-RevIN minimizes
\begin{equation}
R_{\mathrm{TP}}
\left(
\boldsymbol{\theta}
\right)
=
\sum_{a}
\pi_a
R_a
\left(
\boldsymbol{\theta}
\right).
\label{eq:tp_asset_risk}
\end{equation}
Up to a positive global constant, Equation~\eqref{eq:raw_asset_risk}
is equivalent to replacing the declared asset distribution $\pi_a$
with
\begin{equation}
q_a
=
\frac{
\pi_a s_a^2
}{
\sum_b \pi_b s_b^2
}.
\label{eq:effective_asset_distribution}
\end{equation}
Original-space RevIN therefore changes the effective optimization
priority toward assets with larger context scales. TP-RevIN preserves
the declared sampling weights and removes this implicit
scale-dependent redistribution.

This distinction is particularly important for cryptocurrency panels,
where assets can span many orders of magnitude. Under comparable
normalized forecasting errors and model sensitivities, a high-scale
asset can dominate the aggregate parameter update even when low-scale
assets are sampled equally often. TP-RevIN does not make all assets
equally difficult or guarantee equal gradient norms; differences in
forecastability, residual structure, sampling frequency, and model
sensitivity remain. It specifically removes the multiplicative
weighting introduced solely by the RevIN scale and thereby yields a
training objective that is neutral to positive affine changes in the
numerical units of each asset. Further proofs based on Jacobian and
Hessian analyses are provided in
Appendix~\ref{sec:mathematical_analysis}.

\begin{table*}[t]
\centering

\renewcommand{\arraystretch}{1.1}

\begin{tabular}{ll cc gg}
\toprule
\multirow{2}{*}{\textbf{Model}} & \multirow{2}{*}{\textbf{Coin}} & \multicolumn{2}{c}{\textbf{RevIN}} & \multicolumn{2}{c}{\textbf{TP-RevIN}} \\
\cmidrule(lr){3-4} \cmidrule(lr){5-6}
& & \textbf{Fixed Epsilon} & \textbf{Dynamic Epsilon} & \textbf{Fixed Epsilon} & \textbf{Dynamic Epsilon} \\
\midrule
\multirow{5}{*}{Time-MoE} 
& BTC  & 1.15 / 2.38 &  1.16 / 2.38 &  1.12 / 2.23 & 1.16 / 2.27 \\
& ETH  & 1.94 / 4.20 & 1.96 / 4.16 & 1.85 / 3.58 & 1.94 / 3.59 \\
& SHIB & 17362.6 / 12894.9 & 1.89 / 4.12 & 67.24 / 73.13 & 1.87 / 3.51 \\
& DOGE & 2.19 / 4.71 & 2.21 / 4.68 & 2.07 / 4.26 &  2.18 / 4.17 \\
& ADA  & 2.47 / 5.61 & 2.48 / 5.61 & 2.37 / 4.93 & 2.47 / 4.95 \\
\midrule
\multirow{5}{*}{Timer-XL} 
& BTC  & 0.93 / 1.85 & 0.94 / 1.83 & 0.93 / 1.72 & 0.91 / 1.73 \\
& ETH  & 1.61 / 3.19 & 1.65 / 3.26 & 1.67 / 2.81 & 1.61 / 2.88 \\
& SHIB & 181.5 / 3003.3 & 1.48 / 3.20 & 71.2 / 146.3 & 1.40 / 2.74 \\
& DOGE & 1.72 / 3.74 & 1.71 / 3.69 & 1.73 / 3.31 & 1.67 / 3.25 \\
& ADA  & 1.85 / 4.06 & 1.86 / 4.13 & 1.82 / 3.69 & 1.81 / 3.81 \\
\midrule
\multirow{5}{*}{Timer} 
& BTC  & 0.88 / 1.77 & 0.88 / 1.80 & 0.85 / 1.68 & 0.85 / 1.65 \\
& ETH  & 1.44 / 3.02 & 1.45 / 3.00 & 1.40 /  2.70 & 1.39 / 2.65 \\
& SHIB & 471.4 / 5346.9 & 1.45 / 2.97 & 179.0 / 355.4 & 1.38 / 2.72 \\
& DOGE & 1.63 / 3.48 & 1.64 / 3.42 & 1.58 / 3.20 & 1.56 / 3.12 \\
& ADA  & 1.89 / 4.13 & 1.90 / 3.94 & 1.87 / 3.74 & 1.86 / 3.80 \\
\bottomrule
\end{tabular}
\caption{Comparison of Fixed vs. Dynamic Epsilon Normalization (Values report Horizon 5 / 30 for MAPE).}
\label{tab:epsilon_comparison}
\end{table*}

\begin{table*}[ht]
    \centering
    
    \resizebox{\textwidth}{!}{%
        \begin{tabular}{lcccccccccc} 
            \toprule
            \multirow{2}{*}{\textbf{Model}} & \multirow{2.5}{*}{\begin{tabular}[c]{@{}c@{}}\textbf{Time}\\\textbf{Frame}\end{tabular}} & \multirow{2.5}{*}{\textbf{Horizon}} & \multicolumn{2}{c}{\textbf{RevIN (CI)}} & \multicolumn{2}{c}{\textbf{TP-RevIN (CI)}} & \multicolumn{2}{c}{\textbf{TP-RevIN (CD)}} & \multicolumn{2}{c}{\textbf{Unconstrained Space}} \\
            \cmidrule(lr){4-5} \cmidrule(lr){6-7} \cmidrule(lr){8-9} \cmidrule(lr){10-11}
             & & & \textbf{MAE} & \textbf{PHY} & \textbf{MAE} & \textbf{PHY} & \textbf{MAE} & \textbf{PHY} & \textbf{MAE} & \textbf{PHY} \\
            \midrule
            
            \multirow{6}{*}{Time-MoE} 
             & \multirow{3}{*}{30m} 
             & 5  & 47.08 & 0.047 & 45.90 & 1.0e-6 & 49.82 & 1.6e-3 & 459.66 & 0.0 \\
             & & 15 & 78.63 & 2.20 & 72.64 & 5.0e-7 & 80.99 & 1.5e-3 & 452.61 & 0.0 \\
             & & 30 & 100.07 & 6.08 & 91.81 & 3.6e-5 & 101.23 & 7.9e-4 & 517.42 & 0.0 \\
             \cmidrule(lr){2-11} 
             & \multirow{3}{*}{1h}  
             & 5  & 68.28 & 0.161 & 66.91 & 1.0e-6 & 69.25 & 2.7e-3 & 723.18 & 0.0 \\
             & & 15 & 112.06 & 1.14 & 109.06 & 8.0e-6 & 117.69 & 7.8e-3 & 764.56 & 0.0 \\
             & & 30 & 117.92 & 6.14 & 130.28 & 4.3e-5 & 141.20 & 8.9e-3 & 768.84 & 0.0 \\
            \midrule
            
            \multirow{6}{*}{Timer-XL} 
             & \multirow{3}{*}{30m} 
             & 5  & 34.26 & 5.4e-2 & 34.77 & 2.0e-6 & 37.09 & 2.9e-5 & 3307.53 & 0.0 \\
             & & 15 & 55.07 & 2.4e-1 & 54.04 & 1.9e-3 & 55.71 & 2.0e-6 & 1285.78 & 0.0 \\
             & & 30 & 73.50 & 1.6e-1 & 71.19 & 6.9e-4 & 106.98 & 2.0e-4 & 1520.53 & 0.0 \\
             \cmidrule(lr){2-11} 
             & \multirow{3}{*}{1h}  
             & 5  & 52.15 & 8.7e-2 & 52.10 & 1.0e-7 & 54.91 & 7.4e-5 & 1492.80 & 0.0 \\
             & & 15 & 85.10 & 4.4e-1 & 82.53 & 9.6e-5 & 84.85 & 4.5e-4 & 801.78 & 0.0 \\
             & & 30 & 110.69 & 3.1e-1 & 105.5 & 5.0e-7 & 106.98 & 2.0e-4 & 1247.36 & 0.0 \\
            \bottomrule
        \end{tabular}%
    }

\caption{Analysis of training models using auxiliary physics loss on different methods}
\label{tab:physics}
\end{table*}

\subsection{Normalized-Space OHLC Constraint Loss}
\label{subsec:physics_loss}

Under TP-RevIN, both the forecasting objective and the auxiliary OHLC
constraint loss are evaluated in normalized coordinates. Let
$\widehat{\mathbf{Z}}_n$ denote the normalized prediction, with
components
$\widehat{O}_{n,\tau}$, $\widehat{H}_{n,\tau}$,
$\widehat{L}_{n,\tau}$, and $\widehat{C}_{n,\tau}$ at forecast step
$\tau$. Defining $[x]_{+}=\max(0,x)$, the normalized-space constraint
loss is
\begin{equation}
\begin{aligned}
\mathcal{L}_{\mathrm{phy}}
=
\frac{1}{NH}
\sum_{n=1}^{N}
\sum_{\tau=1}^{H}
\Bigl(
&
[\widehat{O}_{n,\tau}-\widehat{H}_{n,\tau}]_{+}
+
[\widehat{C}_{n,\tau}-\widehat{H}_{n,\tau}]_{+}
\\
&+
[\widehat{L}_{n,\tau}-\widehat{O}_{n,\tau}]_{+}
+
[\widehat{L}_{n,\tau}-\widehat{C}_{n,\tau}]_{+}
\\
&+
[\widehat{L}_{n,\tau}-\widehat{H}_{n,\tau}]_{+}
\Bigr).
\end{aligned}
\label{eq:normalized_physics_loss}
\end{equation}
The complete TP-RevIN training objective is therefore
\begin{equation}
\mathcal{L}_{\mathrm{total}}
=
\lambda_{\mathrm{MSE}}
\mathcal{L}_{\mathrm{TP}}
+
\lambda_{\mathrm{phy}}
\mathcal{L}_{\mathrm{phy}},
\label{eq:total_training_objective}
\end{equation}
where both terms are computed before inverse normalization. This
design prevents the auxiliary loss from reintroducing the
scale-dependent weighting removed by TP-RevIN. Additional mathematical
results concerning normalized-space OHLC constraints are provided in
Appendix~\ref{sec:mathematical_analysis}.

\subsection{CryptoL Framework}
\label{subsec:cryptol_framework}

CryptoL combines dynamic-epsilon normalization, TP-RevIN, and the
normalized-space OHLC constraint loss within a unified multivariate
cryptocurrency forecasting framework. The input context and target are
first transformed using either CI or CD statistics together with the
dynamic epsilon in Equation~\eqref{eq:dynamic_epsilon}. The forecasting
backbone predicts the future sequence in normalized coordinates, where
both the MSE objective and the auxiliary constraint loss are evaluated
according to Equation~\eqref{eq:total_training_objective}. Inverse
normalization is applied only after optimization to recover predictions
in their original price units.

The selection between CI and CD remains configurable because the two
schemes provide different practical trade-offs. CD preserves OHLC
ordering under normalization, whereas CI can
provide stronger channelwise forecasting accuracy. Accordingly, both CI and CD variants of
CryptoL achieve low physical-consistency errors when combined with the
constraint loss, allowing the normalization axis to be selected
according to the dataset, backbone, and target evaluation criterion.

\section{Experiments}
\label{sec:experiments}

\subsection{Dataset}
\label{subsec:dataset}

We construct a large-scale cryptocurrency forecasting dataset containing approximately \(15.5\) million historical OHLC observations from \(16\) assets. The data were collected through the Binance API at multiple temporal resolutions, enabling evaluation across different forecasting time frames and horizons. The selected assets exhibit extreme scale heterogeneity, with quoted values ranging from approximately \(10^{-7}\) for low-valued assets such as PEPE to above \(10^{5}\) for high-valued assets such as BTC. This broad dynamic range provides a suitable setting for evaluating normalization behavior, cross-asset learning, and the scale-balancing properties of TP-RevIN. Complete dataset details, including temporal coverage, row counts, and asset-specific volatility distributions, are provided in Appendix~\ref{sec:dataset}.

\subsection{Training Configuration}
\label{subsec:training_configuration}

We evaluate CryptoL using three decoder-only time-series forecasting backbones: Timer \cite{timer} , Timer-XL \cite{timerxl}, and Time-MoE \cite{timemoe}. All models are optimized using AdamW with a linearly scheduled learning rate initialized at \(10^{-5}\). Each model is trained for one epoch using \(90\%\) of the available data, while the remaining \(10\%\) is reserved for evaluation. All reported forecasting metrics are computed after restoring predictions to the original physical space, ensuring a consistent and fair comparison among normalization and training configurations.

Each experiment is conducted separately for a specific time frame and forecast horizon. Within each configuration, however, the model is trained jointly on all \(16\) assets rather than independently for each cryptocurrency. This protocol evaluates whether a shared forecasting backbone can learn transferable temporal and cross-asset patterns despite substantial differences in asset scale, volatility, and market behavior.

\subsection{Hardware and Software Environment}
\label{subsec:hardware_environment}

All experiments are conducted on eight TPU v5e accelerators, each providing \(16\,\mathrm{GB}\) of accelerator memory, for a total distributed memory capacity of \(128\,\mathrm{GB}\). The implementation is developed using JAX  \cite{jax} and Flax \cite{flax}.\footnote{The experiments use JAX version 0.9.2 and Flax version 0.12.6.} Both model parameters and intermediate activations are represented in FP32 throughout training.

To distribute the training workload, we employ fully sharded data-parallel training across all eight TPUs. The model parameters, optimizer states, and input data are sharded over the complete accelerator mesh, reducing the memory allocated to each device and enabling efficient training of the selected forecasting backbones.

\section{Results and Analysis}

We begin by evaluating the predictive performance of different normalization strategies, as summarized in Table~\ref{tab:main_table} and Table~\ref{tab:san_fan_tp}. Without any RevIN \cite{revin}, the models struggle to converge effectively, resulting in extremely high MSE values due to the immense scale discrepancies among different cryptocurrency assets. Introducing standard CI and CD RevIN significantly mitigates this problem. However, TP-RevIN \cite{gtt, revinanalysis} formulations consistently yield further reductions in both MSE and MAE across various forecasting horizons and model backbones. When comparing TP-RevIN to alternative temporal normalization schemes such as FAN \cite{fan} and SAN \cite{san} in Table~\ref{tab:san_fan_tp}, we observe that both FAN and SAN yield substantially higher MAE and Mean Absolute Percentage Error (MAPE) values. This indicates that while those methods are designed to handle non-stationarity, they struggle to balance the extreme cross-asset scale differences present in joint cryptocurrency panels. In contrast, TP-RevIN preserves representation stability across diverse scales and maintains lower physical violation (PHY) rates.

To evaluate the impact of numerical stabilization across assets of varying price scales, we analyze the fixed versus dynamic epsilon formulations in Table~\ref{tab:epsilon_comparison}. For high-value assets such as Bitcoin (BTC) and Ethereum (ETH), the choice between a fixed and dynamic stabilizer ($\epsilon$) has negligible impact on the forecasting performance. However, for low-value assets like Shiba Inu (SHIB), which has an average price scale on the order of $10^{-5}$, the fixed epsilon stabilizer ($\epsilon^{\text{fix}} = 10^{-5}$) introduces severe numerical distortions during normalization. This distortion results in highly elevated MAPE values, such as $17362.6$ for Time-MoE under standard RevIN. Utilizing the scale-adaptive dynamic epsilon ($\epsilon^{\text{dyn}}$) resolves this numerical instability, reducing the SHIB MAPE to $1.87$ / $3.51$ under TP-RevIN and stabilizing the optimization process without compromising the predictive accuracy of the high-value assets.

Finally, we analyze the effectiveness of the auxiliary physics-informed constraint loss, as presented in Table~\ref{tab:physics}. The "Unconstrained Space" method \cite{uncons} mathematically guarantees zero physical violations ($\text{PHY} = 0.0$) through its structural design. However, this rigid constraint comes at a severe cost to overall predictive accuracy, leading to MAE values that are orders of magnitude higher than those of alternative approaches (for example, an MAE of $3307.53$ for Timer-XL on a 30m timeframe at horizon 5, compared to $34.77$ for TP-RevIN (CI)). Conversely, the proposed TP-RevIN models trained with the auxiliary physics loss find a more balanced trade-off. They reduce physical candlestick violations to near-zero levels while keeping the MAE highly competitive.

\section{Conclusion}

In this work, we presented CryptoL, a framework for multivariate cryptocurrency forecasting that addresses cross-asset scale heterogeneity, numerical instability, and physical candlestick constraint violations. By combining Two-Phase RevIN, dynamic epsilon stabilization, and a normalized-space auxiliary physics loss, the framework balances optimization across highly diverse price scales while encouraging structural validity. Empirical evaluations indicate that this integrated approach offers a practical compromise between forecasting accuracy and physical consistency, providing a stable foundation for training shared models on heterogeneous financial time series.

\bibliography{aaai2027}

\onecolumn
\appendix
\setcounter{secnumdepth}{1}

\section{Discussion}
\label{sec:discussion}

Cryptocurrency OHLC forecasting combines several difficulties that are
usually studied separately in general-purpose time-series forecasting.
First, cryptocurrency assets exhibit extreme numerical heterogeneity:
the quoted value of one asset may be below \(10^{-7}\), while another
may exceed \(10^{5}\). A shared forecasting model must therefore learn
from series whose numerical scales differ by more than twelve orders of
magnitude. Second, OHLC observations are structurally constrained. A
valid candlestick must satisfy
\begin{equation}
H \geq \max(O,C),
\qquad
L \leq \min(O,C),
\label{eq:discussion_ohlc_constraints}
\end{equation}
and a model may produce an invalid financial object even when its
aggregate pointwise forecasting error is small.

Reversible instance normalization reduces temporal distribution shift
by transforming each context into normalized coordinates before it is
processed by the forecasting backbone. However, normalization is not a
single unambiguous operation in multivariate OHLC forecasting. When
each channel is normalized independently, the open, high, low, and
close channels undergo different affine transformations. Their
original cross-channel ordering is therefore not guaranteed to remain
valid in normalized coordinates. In contrast, when all four channels
share one positive affine transformation, their pairwise order is
preserved exactly. This distinction motivates the configurable
channel-independent and channel-dependent components of CryptoL.

\begin{figure*}[t]
    \centering
    \includegraphics[width=\textwidth]{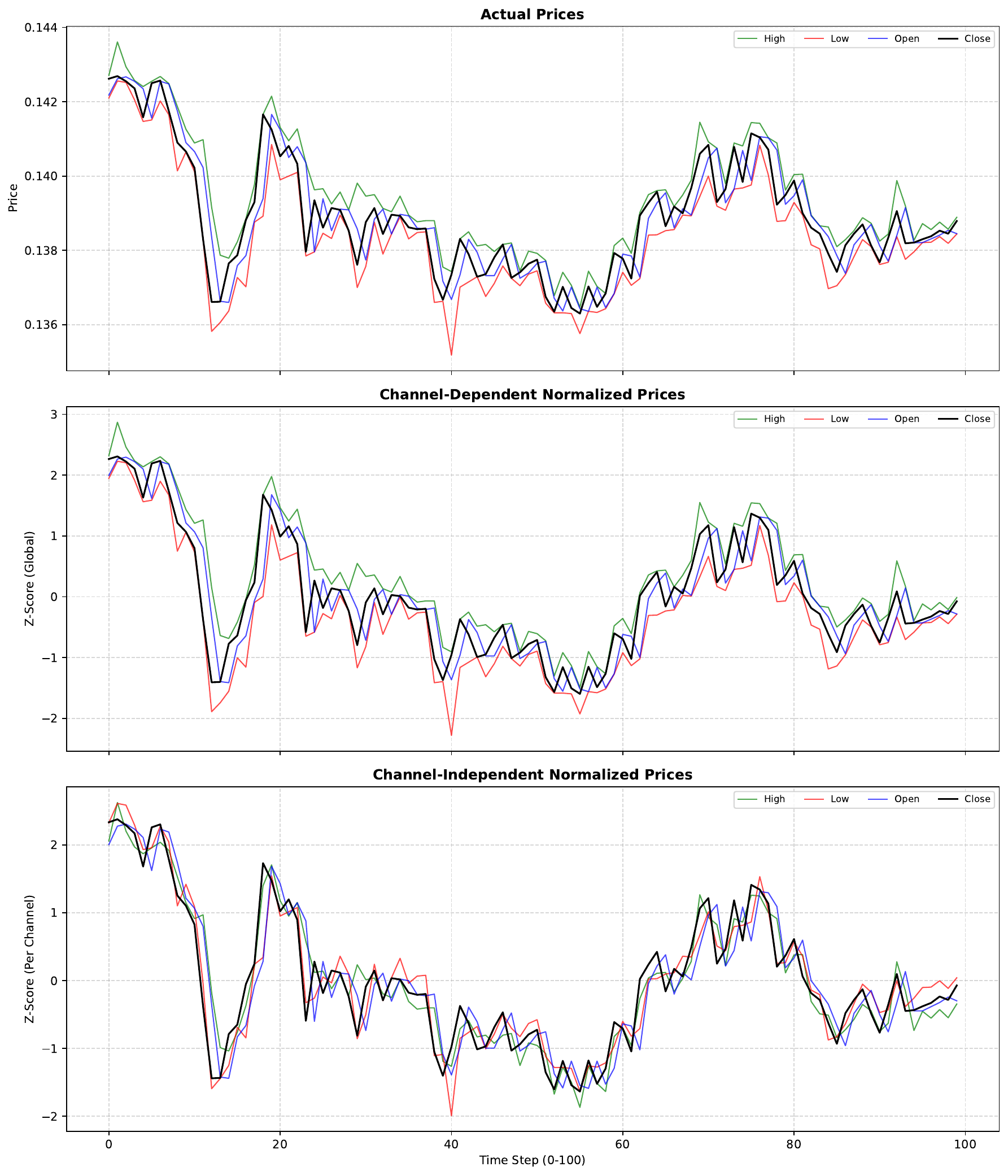}
    \caption{Difference between channel-dependent and channel-independent normalization on OHLC data and how they perform on order preservation.}
    \label{fig:ohlc}
\end{figure*}

A second issue concerns the coordinate system in which optimization is
performed. Under conventional RevIN-style training, the model predicts
in normalized coordinates, but the prediction may be inverse
normalized before mean-squared error is evaluated. This operation
reintroduces the context scale into the empirical objective. For a
sample with scalar context scale \(s_n\), the raw-space residual is
\(s_n\) times the normalized residual, and its squared loss is therefore
weighted by \(s_n^2\). In a shared multi-asset model, this weighting can
alter the aggregate gradient direction and the local curvature of the
optimization problem. TP-RevIN avoids this effect by evaluating the
forecasting objective directly in normalized target coordinates.

CryptoL additionally evaluates its OHLC consistency penalty in
normalized coordinates. Consequently, both the forecasting objective
and the auxiliary structural objective are prevented from inheriting
raw numerical scale. Under CD normalization, normalized-space OHLC
constraints are exactly equivalent to the original-space constraints
because all channels share the same positive affine transformation.
Under CI normalization, this exact equivalence does not hold; the
constraint term should instead be interpreted as a representation-space
regularizer, and financial validity must still be evaluated after
inverse normalization. Empirically, the auxiliary loss substantially
reduces original-space violations under both normalization schemes,
allowing the normalization axis to remain a configurable design choice.

The final component is the stabilizing term used in the normalization
denominator. A fixed epsilon has an absolute numerical scale and may
dominate the empirical variance of extremely low-valued assets while
being negligible for high-valued assets. CryptoL therefore studies a
dynamic epsilon proportional to the squared context mean. This choice
is approximately equivariant under positive changes of numerical units
and makes the regularization strength depend primarily on relative
rather than absolute scale. The following sections formalize these
claims and delimit their assumptions.

\section{Mathematical Analysis}
\label{sec:mathematical_analysis}

Let
\begin{equation}
\mathcal{D}
=
\left\{
\left(
\mathbf{X}_n,
\mathbf{Y}_n
\right)
\right\}_{n=1}^{N},
\qquad
\mathbf{X}_n\in\R^{L\times 4},
\qquad
\mathbf{Y}_n\in\R^{H\times 4},
\label{eq:supp_dataset}
\end{equation}
denote a set of OHLC forecasting samples. The four channels are ordered
as
\begin{equation}
\mathbf{X}_{n,t,:}
=
\left(
O_{n,t},
H_{n,t},
L_{n,t},
C_{n,t}
\right).
\label{eq:supp_channel_order}
\end{equation}
All normalization statistics are assumed to be measurable functions of
the observed context \(\mathbf{X}_n\) only. They are treated as
constants with respect to the forecasting parameters
\(\bm{\theta}\).

For the optimization analysis, the forecast horizon and channel
dimensions are vectorized into
\begin{equation}
d = 4H,
\qquad
\mathbf{z}_n,\widehat{\mathbf{z}}_n\in\R^{d}.
\label{eq:vectorized_dimension}
\end{equation}

\subsection{Order Preservation Under CD and CI Normalization}
\label{subsec:order_preservation}

\subsubsection{Channel-dependent normalization}

Under CD normalization, one shared location and one shared positive
scale are computed from the complete context:
\begin{equation}
\mu_n^{\CD}
=
\frac{1}{4L}
\sum_{t=1}^{L}
\sum_{c=1}^{4}
X_{n,t,c},
\label{eq:supp_cd_mean}
\end{equation}
\begin{equation}
v_n^{\CD}
=
\frac{1}{4L}
\sum_{t=1}^{L}
\sum_{c=1}^{4}
\left(
X_{n,t,c}
-
\mu_n^{\CD}
\right)^2,
\label{eq:supp_cd_variance}
\end{equation}
and
\begin{equation}
s_n^{\CD}
=
\sqrt{
v_n^{\CD}
+
\epsilon_n
}
>0.
\label{eq:supp_cd_scale}
\end{equation}
The shared affine transformation is
\begin{equation}
T_n^{\CD}(x)
=
\frac{x-\mu_n^{\CD}}{s_n^{\CD}}.
\label{eq:supp_cd_transform}
\end{equation}

\begin{theorem}[Order preservation under CD normalization]
\label{thm:cd_order_preservation}
Let \(x_1,x_2\in\R\), and let \(T_n^{\CD}\) be defined by
\cref{eq:supp_cd_transform} with \(s_n^{\CD}>0\). Then
\begin{equation}
x_1\geq x_2
\quad\Longleftrightarrow\quad
T_n^{\CD}(x_1)
\geq
T_n^{\CD}(x_2).
\label{eq:cd_pairwise_equivalence}
\end{equation}
Consequently, every valid OHLC observation remains valid after CD
normalization.
\end{theorem}

\begin{proof}
Subtracting the two transformed values gives
\begin{equation}
T_n^{\CD}(x_1)
-
T_n^{\CD}(x_2)
=
\frac{x_1-x_2}{s_n^{\CD}}.
\label{eq:cd_difference}
\end{equation}
Because \(s_n^{\CD}>0\),
\begin{equation}
\sign
\left(
T_n^{\CD}(x_1)
-
T_n^{\CD}(x_2)
\right)
=
\sign(x_1-x_2).
\label{eq:cd_sign_preservation}
\end{equation}
Therefore \(x_1-x_2\geq0\) if and only if the transformed difference
is nonnegative, proving
\cref{eq:cd_pairwise_equivalence}.

Applying the result to the OHLC inequalities gives
\begin{equation}
H\geq O
\Longrightarrow
\widetilde H^{\CD}
\geq
\widetilde O^{\CD},
\label{eq:cd_high_open}
\end{equation}
\begin{equation}
H\geq C
\Longrightarrow
\widetilde H^{\CD}
\geq
\widetilde C^{\CD},
\label{eq:cd_high_close}
\end{equation}
\begin{equation}
L\leq O
\Longrightarrow
\widetilde L^{\CD}
\leq
\widetilde O^{\CD},
\label{eq:cd_low_open}
\end{equation}
and
\begin{equation}
L\leq C
\Longrightarrow
\widetilde L^{\CD}
\leq
\widetilde C^{\CD}.
\label{eq:cd_low_close}
\end{equation}
Hence membership in the OHLC feasible set is preserved. The figure \ref{fig:ohlc} shows how CD and CI perform different on order preservation for OHLC data.
\end{proof}

\begin{corollary}[Equivalence of CD normalized and raw constraints]
\label{cor:cd_hinge_equivalence}
For every \(s_n^{\CD}>0\),
\begin{equation}
\relu{
\widetilde O^{\CD}
-
\widetilde H^{\CD}
}
=
\frac{1}{s_n^{\CD}}
\relu{
O-H
}.
\label{eq:cd_hinge_open_high}
\end{equation}
Analogous identities hold for every pairwise OHLC constraint.
Therefore, the zero set of the normalized-space CD physics loss is
identical to the zero set of the corresponding raw-space OHLC
constraint loss.
\end{corollary}

\begin{proof}
Using the shared affine transformation,
\begin{equation}
\widetilde O^{\CD}
-
\widetilde H^{\CD}
=
\frac{O-H}{s_n^{\CD}}.
\label{eq:cd_normalized_violation}
\end{equation}
For \(a>0\), the positive-part function satisfies
\begin{equation}
\relu{ax}
=
a\relu{x}.
\label{eq:positive_homogeneity_relu}
\end{equation}
Setting \(a=1/s_n^{\CD}\) proves
\cref{eq:cd_hinge_open_high}.
\end{proof}

\subsubsection{Channel-independent normalization}

Under CI normalization, each channel \(c\in\{O,H,L,C\}\) has its own
statistics:
\begin{equation}
\mu_{n,c}^{\CI}
=
\frac{1}{L}
\sum_{t=1}^{L}
X_{n,t,c},
\label{eq:supp_ci_mean}
\end{equation}
\begin{equation}
v_{n,c}^{\CI}
=
\frac{1}{L}
\sum_{t=1}^{L}
\left(
X_{n,t,c}
-
\mu_{n,c}^{\CI}
\right)^2,
\label{eq:supp_ci_variance}
\end{equation}
and
\begin{equation}
s_{n,c}^{\CI}
=
\sqrt{
v_{n,c}^{\CI}
+
\epsilon_{n,c}
}
>0.
\label{eq:supp_ci_scale}
\end{equation}
The channel-specific transformation is
\begin{equation}
T_{n,c}^{\CI}(x)
=
\frac{
x-\mu_{n,c}^{\CI}
}{
s_{n,c}^{\CI}
}.
\label{eq:supp_ci_transform}
\end{equation}

\begin{proposition}[CI normalization does not guarantee OHLC order]
\label{prop:ci_no_guarantee}
There exist valid OHLC values satisfying \(H>O\) for which
\begin{equation}
T_{n,H}^{\CI}(H)
<
T_{n,O}^{\CI}(O).
\label{eq:ci_reversed_order}
\end{equation}
Therefore, validity in original coordinates does not imply validity in
CI-normalized coordinates.
\end{proposition}

\begin{proof}
Consider
\begin{equation}
O=10,
\qquad
H=11,
\label{eq:ci_counterexample_values}
\end{equation}
so that \(H>O\). Let the channel statistics be
\begin{equation}
\mu_O=9,
\qquad
s_O=0.5,
\qquad
\mu_H=10.8,
\qquad
s_H=0.2.
\label{eq:ci_counterexample_statistics}
\end{equation}
The normalized values are
\begin{equation}
\widetilde O^{\CI}
=
\frac{10-9}{0.5}
=
2,
\label{eq:ci_counterexample_open}
\end{equation}
and
\begin{equation}
\widetilde H^{\CI}
=
\frac{11-10.8}{0.2}
=
1.
\label{eq:ci_counterexample_high}
\end{equation}
Hence
\begin{equation}
H>O
\qquad\text{but}\qquad
\widetilde H^{\CI}
<
\widetilde O^{\CI}.
\label{eq:ci_counterexample_conclusion}
\end{equation}
The original ordering is therefore not guaranteed.
\end{proof}

\begin{remark}
The proposition does not state that CI normalization always destroys
OHLC ordering. It states only that no universal implication of the form
\begin{equation}
H\geq O
\Longrightarrow
\widetilde H^{\CI}
\geq
\widetilde O^{\CI}
\label{eq:ci_invalid_implication}
\end{equation}
can be established without additional restrictions on the
channel-specific statistics.
\end{remark}

\begin{remark}
When the physics loss is evaluated in CI-normalized coordinates, it
enforces a representation-space ordering rather than an exactly
equivalent raw-space ordering. Original-space OHLC validity must
therefore be evaluated after inverse normalization.
\end{remark}

\subsection{Mathematical Analysis of Dynamic Epsilon}
\label{subsec:dynamic_epsilon_analysis}

Let \(x_1,\ldots,x_L\) denote one context sequence with empirical mean
and variance
\begin{equation}
\mu
=
\frac{1}{L}
\sum_{t=1}^{L}
x_t,
\label{eq:epsilon_mean}
\end{equation}
\begin{equation}
v
=
\frac{1}{L}
\sum_{t=1}^{L}
(x_t-\mu)^2.
\label{eq:epsilon_variance}
\end{equation}

The fixed-epsilon normalization is
\begin{equation}
\widetilde x_t^{\mathrm{fix}}
=
\frac{
x_t-\mu
}{
\sqrt{
v+\epsilon_0
}
},
\qquad
\epsilon_0=10^{-5}.
\label{eq:fixed_epsilon_normalization}
\end{equation}
CryptoL uses
\begin{equation}
\epsilon^{\mathrm{dyn}}
=
\lambda
\left(
\mu^2+\delta
\right),
\qquad
\lambda=10^{-5},
\qquad
\delta=10^{-12},
\label{eq:dynamic_epsilon_definition}
\end{equation}
and
\begin{equation}
\widetilde x_t^{\mathrm{dyn}}
=
\frac{
x_t-\mu
}{
\sqrt{
v+
\lambda(\mu^2+\delta)
}
}.
\label{eq:dynamic_epsilon_normalization}
\end{equation}

\begin{proposition}[Non-equivariance of fixed epsilon]
\label{prop:fixed_epsilon_nonequivariant}
Let
\begin{equation}
x_t'=ax_t,
\qquad
a>0.
\label{eq:positive_rescaling}
\end{equation}
Then fixed-epsilon normalization satisfies
\begin{equation}
\widetilde x_t^{\prime\,\mathrm{fix}}
=
\frac{
x_t-\mu
}{
\sqrt{
v+\epsilon_0/a^2
}
}.
\label{eq:fixed_epsilon_scaled}
\end{equation}
Consequently,
\begin{equation}
\widetilde x_t^{\prime\,\mathrm{fix}}
\neq
\widetilde x_t^{\mathrm{fix}}
\label{eq:fixed_epsilon_noninvariance}
\end{equation}
in general.
\end{proposition}

\begin{proof}
Under \cref{eq:positive_rescaling},
\begin{equation}
\mu'=a\mu,
\qquad
v'=a^2v.
\label{eq:scaled_mean_variance}
\end{equation}
Therefore,
\begin{align}
\widetilde x_t^{\prime\,\mathrm{fix}}
&=
\frac{
ax_t-a\mu
}{
\sqrt{
a^2v+\epsilon_0
}
}
\nonumber\\
&=
\frac{
a(x_t-\mu)
}{
a\sqrt{
v+\epsilon_0/a^2
}
}
\nonumber\\
&=
\frac{
x_t-\mu
}{
\sqrt{
v+\epsilon_0/a^2
}
},
\label{eq:fixed_epsilon_proof}
\end{align}
where \(a>0\) was used. The denominator depends on \(a\), proving the
claim.
\end{proof}

\begin{theorem}[Approximate scale equivariance of dynamic epsilon]
\label{thm:dynamic_epsilon_equivariance}
Under the rescaling in \cref{eq:positive_rescaling}, dynamic-epsilon
normalization satisfies
\begin{equation}
\widetilde x_t^{\prime\,\mathrm{dyn}}
=
\frac{
x_t-\mu
}{
\sqrt{
v+\lambda\mu^2+\lambda\delta/a^2
}
}.
\label{eq:dynamic_epsilon_scaled}
\end{equation}
If \(\delta=0\), then
\begin{equation}
\widetilde x_t^{\prime\,\mathrm{dyn}}
=
\widetilde x_t^{\mathrm{dyn}}
\label{eq:dynamic_epsilon_exact_invariance}
\end{equation}
for every \(a>0\). For \(\delta>0\), the transformation is
approximately equivariant whenever
\begin{equation}
\delta
\ll
\mu^2
\qquad\text{and}\qquad
\frac{\delta}{a^2}
\ll
v+\lambda\mu^2.
\label{eq:dynamic_epsilon_approx_condition}
\end{equation}
\end{theorem}

\begin{proof}
Using \cref{eq:scaled_mean_variance},
\begin{align}
\widetilde x_t^{\prime\,\mathrm{dyn}}
&=
\frac{
a(x_t-\mu)
}{
\sqrt{
a^2v+
\lambda(a^2\mu^2+\delta)
}
}
\nonumber\\
&=
\frac{
a(x_t-\mu)
}{
a
\sqrt{
v+\lambda\mu^2+\lambda\delta/a^2
}
}
\nonumber\\
&=
\frac{
x_t-\mu
}{
\sqrt{
v+\lambda\mu^2+\lambda\delta/a^2
}
},
\label{eq:dynamic_epsilon_proof}
\end{align}
which proves \cref{eq:dynamic_epsilon_scaled}. If \(\delta=0\), the
right-hand side equals
\cref{eq:dynamic_epsilon_normalization}. For \(\delta>0\), the only
scale-dependent difference is the term \(\lambda\delta/a^2\), which is
negligible under \cref{eq:dynamic_epsilon_approx_condition}.
\end{proof}

\begin{proposition}[Dependence on relative variability]
\label{prop:dynamic_epsilon_cv}
Assume \(\mu\neq0\), and define the squared coefficient of variation
as
\begin{equation}
\operatorname{CV}^2
=
\frac{v}{\mu^2}.
\label{eq:coefficient_variation}
\end{equation}
Ignoring the negligible \(\delta\) term, the ratio between the dynamic
regularizer and the empirical variance is
\begin{equation}
\frac{
\epsilon^{\mathrm{dyn}}
}{
v
}
=
\frac{
\lambda
}{
\operatorname{CV}^2
}.
\label{eq:dynamic_epsilon_cv_ratio}
\end{equation}
Thus, for assets with comparable relative volatility, the effect of
dynamic epsilon is approximately independent of their absolute quoted
price.
\end{proposition}

\begin{proof}
From \cref{eq:dynamic_epsilon_definition}, with \(\delta\) omitted,
\begin{equation}
\frac{
\epsilon^{\mathrm{dyn}}
}{
v
}
=
\frac{
\lambda\mu^2
}{
v
}
=
\frac{
\lambda
}{
v/\mu^2
}
=
\frac{
\lambda
}{
\operatorname{CV}^2
}.
\label{eq:dynamic_epsilon_cv_proof}
\end{equation}
\end{proof}

\begin{remark}
Dynamic epsilon is not universally preferable. If
\begin{equation}
v\ll\lambda\mu^2,
\label{eq:dynamic_epsilon_dominance_condition}
\end{equation}
then the regularization term dominates the empirical variance and may
compress variations in nearly constant contexts. Its value should
therefore be examined empirically across assets and volatility regimes.
\end{remark}

\subsection{Jacobian Analysis of Scale-Dominance Mitigation}
\label{subsec:jacobian_analysis}

Let
\begin{equation}
\widetilde{\mathbf{X}}_n
=
\mathcal{N}_n(\mathbf{X}_n)
\label{eq:jacobian_normalized_input}
\end{equation}
denote the normalized context, and let the forecasting backbone produce
\begin{equation}
\widehat{\mathbf{z}}_n
=
f_{\bm{\theta}}
\left(
\widetilde{\mathbf{X}}_n
\right).
\label{eq:jacobian_normalized_prediction}
\end{equation}
The normalized target and normalized residual are
\begin{equation}
\mathbf{z}_n
=
\mathbf{D}_n^{-1}
\left(
\mathbf{y}_n-\bm{\mu}_n
\right),
\qquad
\mathbf{e}_n(\bm{\theta})
=
\widehat{\mathbf{z}}_n-\mathbf{z}_n,
\label{eq:jacobian_normalized_target_residual}
\end{equation}
where \(\mathbf{D}_n\succ0\) is diagonal. The reconstructed prediction
is
\begin{equation}
\widehat{\mathbf{y}}_n
=
\bm{\mu}_n
+
\mathbf{D}_n
\widehat{\mathbf{z}}_n.
\label{eq:jacobian_denormalized_prediction}
\end{equation}

The normalized prediction Jacobian is
\begin{equation}
\mathbf{J}_n
=
\frac{
\partial
\widehat{\mathbf{z}}_n
}{
\partial\bm{\theta}
}
\in\R^{d\times p}.
\label{eq:normalized_prediction_jacobian}
\end{equation}

\begin{theorem}[Residual-Jacobian scaling under original-space RevIN]
\label{thm:jacobian_scale_weighting}
Assume that \(\bm{\mu}_n\) and \(\mathbf{D}_n\) are independent of
\(\bm{\theta}\). Define the original-space residual
\begin{equation}
\mathbf{r}_n^{\Raw}
=
\widehat{\mathbf{y}}_n-\mathbf{y}_n.
\label{eq:raw_residual}
\end{equation}
Then
\begin{equation}
\mathbf{r}_n^{\Raw}
=
\mathbf{D}_n\mathbf{e}_n,
\label{eq:raw_normalized_residual}
\end{equation}
and the Jacobian of the original-space residual is
\begin{equation}
\mathbf{J}_n^{\Raw}
=
\frac{
\partial
\mathbf{r}_n^{\Raw}
}{
\partial\bm{\theta}
}
=
\mathbf{D}_n\mathbf{J}_n.
\label{eq:raw_residual_jacobian}
\end{equation}
Consequently, the original-space MSE gradient contribution is
\begin{equation}
\mathbf{g}_n^{\Raw}
=
\left(
\mathbf{J}_n^{\Raw}
\right)^{\top}
\mathbf{r}_n^{\Raw}
=
\mathbf{J}_n^{\top}
\mathbf{D}_n^{2}
\mathbf{e}_n.
\label{eq:raw_jacobian_gradient}
\end{equation}
In contrast, the TP-RevIN gradient contribution is
\begin{equation}
\mathbf{g}_n^{\TP}
=
\mathbf{J}_n^{\top}
\mathbf{e}_n.
\label{eq:tp_jacobian_gradient}
\end{equation}
\end{theorem}

\begin{proof}
From \cref{eq:jacobian_denormalized_prediction} and the identity
\begin{equation}
\mathbf{y}_n
=
\bm{\mu}_n
+
\mathbf{D}_n\mathbf{z}_n,
\label{eq:raw_target_reconstruction}
\end{equation}
we obtain
\begin{align}
\mathbf{r}_n^{\Raw}
&=
\bm{\mu}_n
+
\mathbf{D}_n\widehat{\mathbf{z}}_n
-
\bm{\mu}_n
-
\mathbf{D}_n\mathbf{z}_n
\nonumber\\
&=
\mathbf{D}_n
\left(
\widehat{\mathbf{z}}_n-\mathbf{z}_n
\right)
\nonumber\\
&=
\mathbf{D}_n\mathbf{e}_n,
\label{eq:raw_residual_proof}
\end{align}
proving \cref{eq:raw_normalized_residual}. Since
\(\mathbf{D}_n\) is independent of \(\bm{\theta}\),
\begin{equation}
\mathbf{J}_n^{\Raw}
=
\frac{
\partial
(\mathbf{D}_n\mathbf{e}_n)
}{
\partial\bm{\theta}
}
=
\mathbf{D}_n
\frac{
\partial\mathbf{e}_n
}{
\partial\bm{\theta}
}
=
\mathbf{D}_n\mathbf{J}_n,
\label{eq:raw_jacobian_proof}
\end{equation}
because the target \(\mathbf{z}_n\) is also independent of
\(\bm{\theta}\).

For a least-squares objective, the per-sample gradient is the transpose
of the residual Jacobian multiplied by the residual:
\begin{align}
\mathbf{g}_n^{\Raw}
&=
\left(
\mathbf{D}_n\mathbf{J}_n
\right)^\top
\left(
\mathbf{D}_n\mathbf{e}_n
\right)
\nonumber\\
&=
\mathbf{J}_n^\top
\mathbf{D}_n^\top
\mathbf{D}_n
\mathbf{e}_n
\nonumber\\
&=
\mathbf{J}_n^\top
\mathbf{D}_n^2
\mathbf{e}_n,
\label{eq:raw_gradient_proof}
\end{align}
where the last equality follows because \(\mathbf{D}_n\) is diagonal
and positive. The TP-RevIN result follows directly from the normalized
residual \(\mathbf{e}_n\) and Jacobian \(\mathbf{J}_n\).
\end{proof}

\begin{corollary}[Scalar-scale gradient weighting]
\label{cor:scalar_jacobian_weighting}
Under CD normalization,
\begin{equation}
\mathbf{D}_n=s_n\mathbf{I}_d,
\label{eq:scalar_scale_matrix}
\end{equation}
and therefore
\begin{equation}
\mathbf{g}_n^{\Raw}
=
s_n^2
\mathbf{J}_n^\top
\mathbf{e}_n
=
s_n^2
\mathbf{g}_n^{\TP}.
\label{eq:scalar_gradient_relation}
\end{equation}
Thus, for equal normalized residual-Jacobian products, original-space
RevIN weights the sample contributions in proportion to \(s_n^2\).
\end{corollary}

\begin{remark}
The relation in \cref{eq:scalar_gradient_relation} does not imply that
every high-scale sample necessarily has a larger gradient. The complete
gradient also depends on
\begin{equation}
\mathbf{J}_n^\top\mathbf{e}_n.
\label{eq:intrinsic_gradient_geometry}
\end{equation}
TP-RevIN removes the explicit scale multiplier but does not equalize
residuals, Jacobians, gradient directions, sample frequencies, or task
difficulty.
\end{remark}

\subsubsection{Stacked Jacobian geometry}

Stack the normalized residuals and Jacobians as
\begin{equation}
\mathbf{e}
=
\col(
\mathbf{e}_1,\ldots,\mathbf{e}_N
),
\qquad
\mathbf{J}
=
\col(
\mathbf{J}_1,\ldots,\mathbf{J}_N
).
\label{eq:stacked_residual_jacobian}
\end{equation}
Define
\begin{equation}
\mathbf{S}
=
\diag(
\mathbf{D}_1,\ldots,\mathbf{D}_N
).
\label{eq:block_scale_matrix}
\end{equation}
Then
\begin{equation}
\mathbf{r}^{\Raw}
=
\mathbf{S}\mathbf{e},
\qquad
\mathbf{J}^{\Raw}
=
\mathbf{S}\mathbf{J},
\label{eq:stacked_raw_geometry}
\end{equation}
whereas TP-RevIN uses
\begin{equation}
\mathbf{r}^{\TP}
=
\mathbf{e},
\qquad
\mathbf{J}^{\TP}
=
\mathbf{J}.
\label{eq:stacked_tp_geometry}
\end{equation}

\begin{theorem}[Conditional Jacobian conditioning result]
\label{thm:jacobian_conditioning}
Assume CD normalization and an orthogonal decomposition
\begin{equation}
\R^p
=
V_1\oplus\cdots\oplus V_A.
\label{eq:orthogonal_parameter_decomposition}
\end{equation}
Suppose the normalized Jacobian of asset \(a\) satisfies
\begin{equation}
\mathbf{J}_a^\top\mathbf{J}_a
=
\mathbf{P}_a,
\label{eq:isometric_asset_jacobian}
\end{equation}
where \(\mathbf{P}_a\) is the orthogonal projector onto \(V_a\), and
\begin{equation}
\sum_{a=1}^{A}
\mathbf{P}_a
=
\mathbf{I}_p.
\label{eq:projector_partition}
\end{equation}
If asset \(a\) has scalar scale \(s_a>0\), then
\begin{equation}
\cond_2
\left(
\mathbf{J}^{\TP}
\right)
=
1,
\label{eq:tp_jacobian_condition_number}
\end{equation}
whereas
\begin{equation}
\cond_2
\left(
\mathbf{J}^{\Raw}
\right)
=
\frac{
s_{\max}
}{
s_{\min}
}.
\label{eq:raw_jacobian_condition_number}
\end{equation}
\end{theorem}

\begin{proof}
Under the stated assumptions,
\begin{equation}
\left(
\mathbf{J}^{\TP}
\right)^\top
\mathbf{J}^{\TP}
=
\sum_{a=1}^{A}
\mathbf{P}_a
=
\mathbf{I}_p.
\label{eq:tp_jacobian_gram}
\end{equation}
All singular values of \(\mathbf{J}^{\TP}\) are therefore equal to
one.

For original-space RevIN,
\begin{equation}
\left(
\mathbf{J}^{\Raw}
\right)^\top
\mathbf{J}^{\Raw}
=
\sum_{a=1}^{A}
s_a^2\mathbf{P}_a.
\label{eq:raw_jacobian_gram}
\end{equation}
For any \(\mathbf{u}\in V_a\),
\begin{equation}
\left(
\sum_{b=1}^{A}
s_b^2\mathbf{P}_b
\right)
\mathbf{u}
=
s_a^2\mathbf{u}.
\label{eq:raw_jacobian_eigenvector}
\end{equation}
Thus the singular values of \(\mathbf{J}^{\Raw}\) are the values
\(s_a\), with multiplicities determined by \(\dim(V_a)\). Their ratio
is \(s_{\max}/s_{\min}\).
\end{proof}

\begin{remark}
\Cref{thm:jacobian_conditioning} is conditional. TP-RevIN removes the
row scaling induced by the RevIN inverse map, but it does not
universally guarantee a smaller Jacobian condition number. Intrinsic
anisotropy in the normalized Jacobians may remain.
\end{remark}

\subsection{Hessian Analysis of Scale-Dominance Mitigation}
\label{subsec:hessian_analysis}

For sample \(n\), define the normalized-space MSE
\begin{equation}
\ell_n^{\TP}
=
\frac{1}{2}
\norm{
\mathbf{e}_n
}_2^2.
\label{eq:tp_sample_loss_hessian}
\end{equation}
The original-space RevIN MSE is
\begin{equation}
\ell_n^{\Raw}
=
\frac{1}{2}
\mathbf{e}_n^\top
\mathbf{W}_n
\mathbf{e}_n,
\qquad
\mathbf{W}_n
=
\mathbf{D}_n^\top\mathbf{D}_n
=
\mathbf{D}_n^2.
\label{eq:raw_sample_loss_hessian}
\end{equation}

Let
\begin{equation}
\widehat z_{n,k}
\label{eq:prediction_component}
\end{equation}
denote output component \(k\), and define its parameter Hessian as
\begin{equation}
\mathbf{H}_{n,k}^{f}
=
\nabla_{\bm{\theta}}^2
\widehat z_{n,k}.
\label{eq:output_component_hessian}
\end{equation}

\begin{theorem}[Exact TP-RevIN and original-space Hessians]
\label{thm:exact_hessians}
Assume \(\mathbf{W}_n\) is independent of \(\bm{\theta}\). Then the
exact TP-RevIN Hessian is
\begin{equation}
\mathbf{H}_n^{\TP}
=
\mathbf{J}_n^\top\mathbf{J}_n
+
\sum_{k=1}^{d}
e_{n,k}
\mathbf{H}_{n,k}^{f}.
\label{eq:exact_tp_hessian}
\end{equation}
The exact original-space RevIN Hessian is
\begin{equation}
\mathbf{H}_n^{\Raw}
=
\mathbf{J}_n^\top
\mathbf{W}_n
\mathbf{J}_n
+
\sum_{k=1}^{d}
\left(
\mathbf{W}_n\mathbf{e}_n
\right)_k
\mathbf{H}_{n,k}^{f}.
\label{eq:exact_raw_hessian}
\end{equation}
Under CD normalization, where
\begin{equation}
\mathbf{W}_n
=
s_n^2\mathbf{I}_d,
\label{eq:scalar_weight_matrix}
\end{equation}
the complete per-sample Hessians satisfy
\begin{equation}
\mathbf{H}_n^{\Raw}
=
s_n^2
\mathbf{H}_n^{\TP}.
\label{eq:scalar_hessian_relation}
\end{equation}
\end{theorem}

\begin{proof}
The TP-RevIN gradient is
\begin{equation}
\nabla_{\bm{\theta}}
\ell_n^{\TP}
=
\mathbf{J}_n^\top
\mathbf{e}_n.
\label{eq:tp_gradient_hessian_proof}
\end{equation}
Differentiating by the product rule gives
\begin{equation}
\nabla_{\bm{\theta}}^2
\ell_n^{\TP}
=
\mathbf{J}_n^\top\mathbf{J}_n
+
\sum_{k=1}^{d}
e_{n,k}
\nabla_{\bm{\theta}}^2
\widehat z_{n,k},
\label{eq:tp_hessian_proof}
\end{equation}
which proves \cref{eq:exact_tp_hessian}.

For original-space RevIN,
\begin{equation}
\nabla_{\bm{\theta}}
\ell_n^{\Raw}
=
\mathbf{J}_n^\top
\mathbf{W}_n
\mathbf{e}_n.
\label{eq:raw_gradient_hessian_proof}
\end{equation}
Differentiating again gives a term from the derivative of the residual
and a term from the derivative of the Jacobian:
\begin{equation}
\nabla_{\bm{\theta}}^2
\ell_n^{\Raw}
=
\mathbf{J}_n^\top
\mathbf{W}_n
\mathbf{J}_n
+
\sum_{k=1}^{d}
\left(
\mathbf{W}_n\mathbf{e}_n
\right)_k
\mathbf{H}_{n,k}^{f},
\label{eq:raw_hessian_proof}
\end{equation}
proving \cref{eq:exact_raw_hessian}.

If \(\mathbf{W}_n=s_n^2\mathbf{I}_d\), then
\begin{align}
\mathbf{H}_n^{\Raw}
&=
s_n^2
\mathbf{J}_n^\top\mathbf{J}_n
+
s_n^2
\sum_{k=1}^{d}
e_{n,k}
\mathbf{H}_{n,k}^{f}
\nonumber\\
&=
s_n^2
\mathbf{H}_n^{\TP},
\label{eq:scalar_hessian_proof}
\end{align}
which proves \cref{eq:scalar_hessian_relation}.
\end{proof}

\begin{remark}
For CI normalization, \(\mathbf{W}_n\) is generally diagonal but not
a scalar multiple of the identity. In that case, the exact Hessian is
given by \cref{eq:exact_raw_hessian}; it is not generally valid to write
\(\mathbf{H}_n^{\Raw}=s_n^2\mathbf{H}_n^{\TP}\) using one scalar
\(s_n\).
\end{remark}

\subsubsection{Gauss--Newton curvature}

The Gauss--Newton matrices retain the positive-semidefinite
first term of the exact Hessians:
\begin{equation}
\mathbf{G}^{\TP}
=
\frac{1}{N}
\sum_{n=1}^{N}
\mathbf{J}_n^\top
\mathbf{J}_n,
\label{eq:tp_gauss_newton}
\end{equation}
and
\begin{equation}
\mathbf{G}^{\Raw}
=
\frac{1}{N}
\sum_{n=1}^{N}
\mathbf{J}_n^\top
\mathbf{W}_n
\mathbf{J}_n.
\label{eq:raw_gauss_newton}
\end{equation}

Under scalar CD scales,
\begin{equation}
\mathbf{G}^{\Raw}
=
\frac{1}{N}
\sum_{n=1}^{N}
s_n^2
\mathbf{J}_n^\top
\mathbf{J}_n.
\label{eq:raw_gauss_newton_scalar}
\end{equation}

\begin{theorem}[Loewner and eigenvalue bounds]
\label{thm:gauss_newton_bounds}
Assume scalar scales satisfying
\begin{equation}
0<s_{\min}
\leq
s_n
\leq
s_{\max}
<\infty.
\label{eq:scale_bounds}
\end{equation}
Then
\begin{equation}
s_{\min}^2
\mathbf{G}^{\TP}
\preceq
\mathbf{G}^{\Raw}
\preceq
s_{\max}^2
\mathbf{G}^{\TP}.
\label{eq:loewner_bounds}
\end{equation}
On any common parameter subspace on which both matrices are positive
definite,
\begin{equation}
\lambda_{\min}
\left(
\mathbf{G}^{\Raw}
\right)
\geq
s_{\min}^2
\lambda_{\min}
\left(
\mathbf{G}^{\TP}
\right),
\label{eq:min_eigenvalue_bound}
\end{equation}
\begin{equation}
\lambda_{\max}
\left(
\mathbf{G}^{\Raw}
\right)
\leq
s_{\max}^2
\lambda_{\max}
\left(
\mathbf{G}^{\TP}
\right),
\label{eq:max_eigenvalue_bound}
\end{equation}
and
\begin{equation}
\cond
\left(
\mathbf{G}^{\Raw}
\right)
\leq
\left(
\frac{s_{\max}}{s_{\min}}
\right)^2
\cond
\left(
\mathbf{G}^{\TP}
\right).
\label{eq:condition_number_upper_bound}
\end{equation}
\end{theorem}

\begin{proof}
For any \(\mathbf{u}\in\R^p\),
\begin{align}
\mathbf{u}^\top
\mathbf{G}^{\Raw}
\mathbf{u}
&=
\frac{1}{N}
\sum_{n=1}^{N}
s_n^2
\norm{
\mathbf{J}_n\mathbf{u}
}_2^2.
\label{eq:raw_quadratic_form}
\end{align}
Applying \cref{eq:scale_bounds} termwise gives
\begin{align}
s_{\min}^2
\frac{1}{N}
\sum_{n=1}^{N}
\norm{
\mathbf{J}_n\mathbf{u}
}_2^2
&\leq
\mathbf{u}^\top
\mathbf{G}^{\Raw}
\mathbf{u}
\nonumber\\
&\leq
s_{\max}^2
\frac{1}{N}
\sum_{n=1}^{N}
\norm{
\mathbf{J}_n\mathbf{u}
}_2^2.
\label{eq:quadratic_form_bounds}
\end{align}
The outer expressions are
\begin{equation}
s_{\min}^2
\mathbf{u}^\top
\mathbf{G}^{\TP}
\mathbf{u}
\label{eq:lower_quadratic_form}
\end{equation}
and
\begin{equation}
s_{\max}^2
\mathbf{u}^\top
\mathbf{G}^{\TP}
\mathbf{u},
\label{eq:upper_quadratic_form}
\end{equation}
which proves the Loewner bounds. The eigenvalue inequalities follow
from the Rayleigh quotient, and their ratio gives
\cref{eq:condition_number_upper_bound}.
\end{proof}

\begin{theorem}[Exact scale-induced conditioning under orthogonal asset subspaces]
\label{thm:orthogonal_hessian_conditioning}
Assume the parameter space decomposes into mutually orthogonal
subspaces \(V_1,\ldots,V_A\), and suppose the normalized Gauss--Newton
operator for asset \(a\) is
\begin{equation}
\mathbf{G}_a
=
\lambda_a\mathbf{P}_a,
\qquad
\lambda_a>0,
\label{eq:asset_gauss_newton_block}
\end{equation}
where \(\mathbf{P}_a\) projects onto \(V_a\). Let the asset sampling
probabilities be \(\pi_a>0\). Then
\begin{equation}
\mathbf{G}^{\TP}
=
\bigoplus_{a=1}^{A}
\pi_a\lambda_a\mathbf{I}_{V_a},
\label{eq:tp_block_gauss_newton}
\end{equation}
and
\begin{equation}
\mathbf{G}^{\Raw}
=
\bigoplus_{a=1}^{A}
\pi_as_a^2\lambda_a\mathbf{I}_{V_a}.
\label{eq:raw_block_gauss_newton}
\end{equation}
Consequently,
\begin{equation}
\cond
\left(
\mathbf{G}^{\TP}
\right)
=
\frac{
\max_a \pi_a\lambda_a
}{
\min_a \pi_a\lambda_a
},
\label{eq:tp_block_condition_number}
\end{equation}
whereas
\begin{equation}
\cond
\left(
\mathbf{G}^{\Raw}
\right)
=
\frac{
\max_a \pi_as_a^2\lambda_a
}{
\min_a \pi_as_a^2\lambda_a
}.
\label{eq:raw_block_condition_number}
\end{equation}
If \(\pi_a\) and \(\lambda_a\) are equal across assets, then
\begin{equation}
\cond
\left(
\mathbf{G}^{\Raw}
\right)
=
\left(
\frac{
s_{\max}
}{
s_{\min}
}
\right)^2,
\qquad
\cond
\left(
\mathbf{G}^{\TP}
\right)
=
1.
\label{eq:exact_scale_conditioning}
\end{equation}
\end{theorem}

\begin{proof}
Because the subspaces are mutually orthogonal, the aggregate
Gauss--Newton matrices are block diagonal. The eigenvalue associated
with \(V_a\) is \(\pi_a\lambda_a\) under TP-RevIN and
\(\pi_as_a^2\lambda_a\) under original-space RevIN. The condition
numbers are therefore the ratios of the largest and smallest block
eigenvalues, proving
\cref{eq:tp_block_condition_number,eq:raw_block_condition_number}.
Equal \(\pi_a\) and \(\lambda_a\) reduce these expressions to
\cref{eq:exact_scale_conditioning}.
\end{proof}

\begin{remark}[Scope of the curvature claim]
\label{rem:hessian_scope}
TP-RevIN removes the explicit scale weighting
\begin{equation}
s_n^2
\mathbf{J}_n^\top
\mathbf{J}_n
\label{eq:removed_scale_curvature}
\end{equation}
from the Gauss--Newton geometry. It does not universally minimize the
condition number of every nonlinear forecasting problem. Scale
weighting may occasionally compensate for pre-existing curvature
anisotropy. The conclusion is therefore that TP-RevIN
removes the component of curvature heterogeneity induced solely by the
RevIN context scales.

\end{remark}

\begin{example}[Scale weighting may leave conditioning unchanged]
\label{ex:aligned_curvature}
Let
\begin{equation}
\mathbf{G}_1
=
\mathbf{G}_2
=
\mathbf{I}_p.
\label{eq:aligned_curvature_matrices}
\end{equation}
Then
\begin{equation}
\mathbf{G}^{\TP}
=
2\mathbf{I}_p,
\qquad
\mathbf{G}^{\Raw}
=
(s_1^2+s_2^2)\mathbf{I}_p.
\label{eq:aligned_curvature_conditioning}
\end{equation}
Both matrices have condition number one, even when
\(s_1\neq s_2\). Thus, scale heterogeneity does not necessarily worsen
conditioning when all samples excite identical isotropic parameter
directions.
\end{example}

\begin{example}[Scale weighting can compensate for intrinsic curvature]
\label{ex:scale_weighting_improves_conditioning}
Let
\begin{equation}
\mathbf{G}_1
=
\begin{bmatrix}
1 & 0\\
0 & 0
\end{bmatrix},
\qquad
\mathbf{G}_2
=
\begin{bmatrix}
0 & 0\\
0 & 100
\end{bmatrix}.
\label{eq:conditioning_counterexample_matrices}
\end{equation}
Then
\begin{equation}
\mathbf{G}^{\TP}
=
\begin{bmatrix}
1 & 0\\
0 & 100
\end{bmatrix},
\qquad
\cond
\left(
\mathbf{G}^{\TP}
\right)
=
100.
\label{eq:tp_conditioning_counterexample}
\end{equation}
Choosing
\begin{equation}
s_1^2=100,
\qquad
s_2^2=1
\label{eq:counterexample_scales}
\end{equation}
gives
\begin{equation}
\mathbf{G}^{\Raw}
=
\begin{bmatrix}
100 & 0\\
0 & 100
\end{bmatrix},
\qquad
\cond
\left(
\mathbf{G}^{\Raw}
\right)
=
1.
\label{eq:raw_conditioning_counterexample}
\end{equation}
This example confirms that universal conditioning superiority cannot be
claimed. TP-RevIN provides scale neutrality, not an unconditional
minimum-condition-number guarantee.
\end{example}

\section{Dataset Details and Volatility Analysis}
\label{sec:dataset}

This section provides a detailed breakdown of the large-scale cryptocurrency forecasting dataset utilized in our empirical evaluations. The dataset consists of historical Open-High-Low-Close (OHLC) observations collected via the Binance API. In total, the dataset contains $15,490,300$ data rows spanning $16$ heterogeneous assets across $7$ distinct temporal resolutions (5-minute, 15-minute, 30-minute, 1-hour, 2-hour, 4-hour, and 1-day intervals), yielding a total of $112$ distinct time-series panels.

\subsection{Dataset Overview and Coverage}

The temporal coverage and active periods for each of the $16$ selected cryptocurrency assets are detailed in Table~\ref{tab:coverage}. The assets cover highly mature coins such as Bitcoin (BTC) and Ethereum (ETH) with histories dating back to 2017, alongside more recently launched tokens such as SUI and TON, which provide evaluation paths for lower-history, high-volatility regimes.

\begin{table}[htbp]
\centering
\caption{Temporal coverage per cryptocurrency asset in the dataset.}
\label{tab:coverage}
\begin{tabular}{lcc}
\toprule
\textbf{Coin} & \textbf{Start Date} & \textbf{End Date} \\
\midrule
ADAUSDT & 2018-04-17 03:30:00 & 2025-06-01 03:25:00 \\
BCHUSDT & 2019-11-28 03:30:00 & 2025-06-01 03:25:00 \\
BNBUSDT & 2017-11-06 03:30:00 & 2025-06-01 03:25:00 \\
BTCUSDT & 2017-08-17 03:30:00 & 2025-06-01 03:15:00 \\
DOGEUSDT & 2019-07-05 03:30:00 & 2025-06-01 03:25:00 \\
ETHUSDT & 2017-08-17 03:30:00 & 2025-06-01 03:25:00 \\
LINKUSDT & 2019-01-16 03:30:00 & 2025-06-01 03:25:00 \\
LTCUSDT & 2017-12-13 03:30:00 & 2025-06-01 03:25:00 \\
PEPEUSDT & 2023-05-05 03:30:00 & 2025-06-01 03:25:00 \\
SHIBUSDT & 2021-05-10 03:30:00 & 2025-06-01 03:25:00 \\
SOLUSDT & 2020-08-11 03:30:00 & 2025-06-01 03:25:00 \\
SUIUSDT & 2023-05-03 03:30:00 & 2025-06-01 03:25:00 \\
TONUSDT & 2024-08-08 03:30:00 & 2025-06-01 03:25:00 \\
TRXUSDT & 2018-06-11 03:30:00 & 2025-06-01 03:25:00 \\
XLMUSDT & 2018-05-31 03:30:00 & 2025-06-01 03:25:00 \\
XRPUSDT & 2018-05-04 03:30:00 & 2025-06-01 03:25:00 \\
\bottomrule
\end{tabular}
\end{table}

The sample density varies across timeframes. Table~\ref{tab:row_counts} outlines the exact number of recorded data rows for each asset across the seven temporal resolutions.

\begin{table}[htbp]
\centering
\caption{Observed data row counts partitioned by coin and time frame.}
\label{tab:row_counts}
\resizebox{\textwidth}{!}{
\begin{tabular}{lrrrrrrr}
\toprule
\textbf{Coin} & \textbf{5m} & \textbf{15m} & \textbf{30m} & \textbf{1h} & \textbf{2h} & \textbf{4h} & \textbf{1d} \\
\midrule
ADAUSDT & 748,156 & 249,388 & 124,699 & 62,357 & 31,188 & 15,602 & 2,602 \\
BCHUSDT & 578,874 & 192,960 & 96,483 & 48,247 & 24,130 & 12,069 & 2,012 \\
BNBUSDT & 794,367 & 264,795 & 132,404 & 66,213 & 33,118 & 16,568 & 2,764 \\
BTCUSDT & 647,474 & 272,542 & 136,277 & 68,149 & 34,086 & 17,052 & 2,845 \\
DOGEUSDT & 620,750 & 206,919 & 103,463 & 51,737 & 25,875 & 12,943 & 2,158 \\
ETHUSDT & 817,609 & 272,542 & 134,789 & 68,149 & 34,086 & 17,052 & 2,845 \\
LINKUSDT & 669,530 & 223,179 & 111,594 & 55,803 & 27,909 & 13,961 & 2,328 \\
LTCUSDT & 783,714 & 261,243 & 130,628 & 65,325 & 32,674 & 16,346 & 2,727 \\
PEPEUSDT & 218,088 & 72,696 & 36,348 & 18,174 & 9,087 & 4,544 & 758 \\
SHIBUSDT & 426,878 & 142,293 & 71,147 & 35,574 & 17,789 & 8,896 & 1,483 \\
SOLUSDT & 505,085 & 168,363 & 84,184 & 42,095 & 21,052 & 10,529 & 1,755 \\
SUIUSDT & 218,736 & 72,912 & 36,456 & 18,228 & 9,114 & 4,557 & 760 \\
TONUSDT & 85,416 & 28,472 & 14,236 & 7,118 & 3,559 & 1,780 & 297 \\
TRXUSDT & 732,226 & 244,078 & 122,044 & 61,030 & 30,525 & 15,271 & 2,547 \\
XLMUSDT & 735,418 & 245,142 & 122,576 & 61,296 & 30,658 & 15,337 & 2,558 \\
XRPUSDT & 743,210 & 247,740 & 123,875 & 61,945 & 30,982 & 15,499 & 2,585 \\
\bottomrule
\end{tabular}
}
\end{table}

\subsection{Empirical Volatility Analysis and Optimization Difficulty}

To understand the optimization challenges across different settings, we analyze the volatility profiles of the dataset. Volatility is defined here as the absolute percentage change in the close price of a candle relative to the preceding candle:
\begin{equation}
    v_t = \frac{|C_t - C_{t-1}|}{C_{t-1}} \times 100\%
\end{equation}
We evaluate the proportion of candles exceeding specified threshold levels ($>1\%$, $>2\%$, $>3\%$, $>4\%$, and $>5\%$).

Table~\ref{tab:total_volatility} reports the aggregate volatility thresholds across all sixteen assets. As the time frame increases, the likelihood of substantial price changes scales non-linearly. For instance, only $0.99\%$ of candles in the $5$-minute resolution experience a price deviation greater than $1\%$, and a mere $0.01\%$ exceed $5\%$. Conversely, in the daily ($1$-day) timeframe, $37.46\%$ of the candles exceed a $1\%$ change, and $11.33\%$ experience fluctuations larger than $5\%$.

\begin{table}[htbp]
\centering
\caption{Aggregate volatility distribution across all assets partitioned by timeframe.}
\label{tab:total_volatility}

\begin{tabular}{lccccc}
\toprule
\textbf{Timeframe} & \textbf{$>1\%$} & \textbf{$>2\%$} & \textbf{$>3\%$} & \textbf{$>4\%$} & \textbf{$>5\%$} \\
\midrule
5m & 0.99\% & 0.15\% & 0.05\% & 0.02\% & 0.01\% \\
15m & 3.10\% & 0.65\% & 0.22\% & 0.10\% & 0.05\% \\
30m & 5.89\% & 1.49\% & 0.56\% & 0.26\% & 0.14\% \\
1h & 9.94\% & 3.07\% & 1.29\% & 0.64\% & 0.35\% \\
2h & 15.45\% & 5.88\% & 2.73\% & 1.47\% & 0.86\% \\
4h & 21.75\% & 10.29\% & 5.43\% & 3.16\% & 1.97\% \\
1d & 37.46\% & 27.27\% & 20.20\% & 15.04\% & 11.33\% \\
\bottomrule
\end{tabular}

\end{table}

This disparity directly aligns with the empirical observations presented in the main paper. As demonstrated in our experimental results, forecasting backbones experience a reduction in predictive accuracy (higher MAE and MAPE metrics) when operating on larger timeframes. The high density of extreme returns on larger temporal resolutions represents a highly volatile dynamical regime with increased variance in the target distributions, presenting a fundamentally more challenging forecasting objective.

\subsection{Asset-Specific Volatility Distributions}

The following tables report the granular, asset-specific volatility distributions across all experimental resolutions. To accommodate the size of the dataset without compilation issues, the records are divided into two distinct parts: Table~\ref{tab:coin_volatility_part1} covers assets from ADAUSDT to LTCUSDT, and Table~\ref{tab:coin_volatility_part2} covers assets from PEPEUSDT to XRPUSDT.

\begin{table}[p]
\centering
\caption{Granular volatility thresholds for assets ADAUSDT through LTCUSDT across experimental resolutions.}
\label{tab:coin_volatility_part1}
\footnotesize
\begin{tabular}{lcccccc}
\toprule
\textbf{Coin} & \textbf{Timeframe} & \textbf{$>1\%$} & \textbf{$>2\%$} & \textbf{$>3\%$} & \textbf{$>4\%$} & \textbf{$>5\%$} \\
\midrule
ADAUSDT & 5m & 0.90\% & 0.11\% & 0.03\% & 0.01\% & 0.01\% \\
ADAUSDT & 15m & 3.35\% & 0.60\% & 0.18\% & 0.07\% & 0.03\% \\
ADAUSDT & 30m & 6.63\% & 1.56\% & 0.51\% & 0.21\% & 0.10\% \\
ADAUSDT & 1h & 11.38\% & 3.40\% & 1.30\% & 0.59\% & 0.30\% \\
ADAUSDT & 2h & 17.65\% & 6.67\% & 3.03\% & 1.55\% & 0.85\% \\
ADAUSDT & 4h & 23.81\% & 11.88\% & 6.14\% & 3.41\% & 2.02\% \\
ADAUSDT & 1d & 39.64\% & 29.10\% & 21.88\% & 16.57\% & 12.61\% \\
\midrule
BCHUSDT & 5m & 0.74\% & 0.10\% & 0.03\% & 0.01\% & 0.01\% \\
BCHUSDT & 15m & 2.61\% & 0.50\% & 0.17\% & 0.06\% & 0.03\% \\
BCHUSDT & 30m & 5.27\% & 1.20\% & 0.46\% & 0.19\% & 0.10\% \\
BCHUSDT & 1h & 9.10\% & 2.68\% & 1.06\% & 0.52\% & 0.27\% \\
BCHUSDT & 2h & 14.70\% & 5.22\% & 2.35\% & 1.21\% & 0.70\% \\
BCHUSDT & 4h & 21.52\% & 9.56\% & 4.89\% & 2.77\% & 1.64\% \\
BCHUSDT & 1d & 37.69\% & 27.00\% & 19.69\% & 14.42\% & 10.89\% \\
\midrule
BNBUSDT & 5m & 0.92\% & 0.18\% & 0.06\% & 0.03\% & 0.01\% \\
BNBUSDT & 15m & 2.72\% & 0.62\% & 0.24\% & 0.11\% & 0.06\% \\
BNBUSDT & 30m & 5.16\% & 1.34\% & 0.55\% & 0.27\% & 0.16\% \\
BNBUSDT & 1h & 8.85\% & 2.72\% & 1.22\% & 0.61\% & 0.36\% \\
BNBUSDT & 2h & 14.09\% & 5.25\% & 2.45\% & 1.36\% & 0.82\% \\
BNBUSDT & 4h & 20.60\% & 9.31\% & 4.80\% & 2.73\% & 1.73\% \\
BNBUSDT & 1d & 36.81\% & 26.28\% & 19.44\% & 14.12\% & 10.46\% \\
\midrule
BTCUSDT & 5m & 0.24\% & 0.03\% & 0.01\% & 0.00\% & 0.00\% \\
BTCUSDT & 15m & 1.46\% & 0.25\% & 0.09\% & 0.04\% & 0.02\% \\
BTCUSDT & 30m & 3.01\% & 0.66\% & 0.21\% & 0.09\% & 0.04\% \\
BTCUSDT & 1h & 5.64\% & 1.48\% & 0.52\% & 0.24\% & 0.12\% \\
BTCUSDT & 2h & 9.74\% & 3.23\% & 1.26\% & 0.58\% & 0.31\% \\
BTCUSDT & 4h & 15.48\% & 6.19\% & 2.87\% & 1.46\% & 0.74\% \\
BTCUSDT & 1d & 33.90\% & 22.15\% & 15.12\% & 10.41\% & 6.93\% \\
\midrule
DOGEUSDT & 5m & 1.30\% & 0.30\% & 0.12\% & 0.07\% & 0.04\% \\
DOGEUSDT & 15m & 3.53\% & 0.94\% & 0.41\% & 0.23\% & 0.14\% \\
DOGEUSDT & 30m & 6.24\% & 1.89\% & 0.84\% & 0.47\% & 0.30\% \\
DOGEUSDT & 1h & 10.03\% & 3.48\% & 1.69\% & 0.93\% & 0.59\% \\
DOGEUSDT & 2h & 15.30\% & 6.15\% & 3.11\% & 1.87\% & 1.23\% \\
DOGEUSDT & 4h & 21.45\% & 10.33\% & 5.59\% & 3.38\% & 2.29\% \\
DOGEUSDT & 1d & 34.96\% & 25.17\% & 18.96\% & 14.51\% & 10.66\% \\
\midrule
ETHUSDT & 5m & 0.66\% & 0.09\% & 0.03\% & 0.01\% & 0.01\% \\
ETHUSDT & 15m & 2.25\% & 0.42\% & 0.12\% & 0.05\% & 0.03\% \\
ETHUSDT & 30m & 4.70\% & 1.06\% & 0.35\% & 0.15\% & 0.07\% \\
ETHUSDT & 1h & 8.41\% & 2.34\% & 0.89\% & 0.41\% & 0.18\% \\
ETHUSDT & 2h & 13.89\% & 4.88\% & 2.12\% & 1.05\% & 0.51\% \\
ETHUSDT & 4h & 20.27\% & 9.23\% & 4.53\% & 2.41\% & 1.33\% \\
ETHUSDT & 1d & 37.62\% & 27.67\% & 20.50\% & 14.94\% & 10.79\% \\
\midrule
LINKUSDT & 5m & 1.22\% & 0.15\% & 0.04\% & 0.02\% & 0.01\% \\
LINKUSDT & 15m & 4.22\% & 0.82\% & 0.23\% & 0.09\% & 0.04\% \\
LINKUSDT & 30m & 8.09\% & 1.94\% & 0.67\% & 0.27\% & 0.12\% \\
LINKUSDT & 1h & 13.38\% & 4.17\% & 1.64\% & 0.76\% & 0.37\% \\
LINKUSDT & 2h & 20.06\% & 8.23\% & 3.64\% & 1.85\% & 1.00\% \\
LINKUSDT & 4h & 26.78\% & 13.93\% & 7.54\% & 4.23\% & 2.44\% \\
LINKUSDT & 1d & 42.16\% & 32.87\% & 25.91\% & 19.55\% & 15.13\% \\
\midrule
LTCUSDT & 5m & 0.85\% & 0.12\% & 0.04\% & 0.02\% & 0.01\% \\
LTCUSDT & 15m & 2.90\% & 0.58\% & 0.18\% & 0.07\% & 0.04\% \\
LTCUSDT & 30m & 5.73\% & 1.39\% & 0.48\% & 0.21\% & 0.10\% \\
LTCUSDT & 1h & 9.99\% & 3.01\% & 1.15\% & 0.52\% & 0.27\% \\
LTCUSDT & 2h & 15.61\% & 5.77\% & 2.56\% & 1.29\% & 0.80\% \\
LTCUSDT & 4h & 22.10\% & 10.31\% & 5.26\% & 2.99\% & 1.79\% \\
LTCUSDT & 1d & 38.59\% & 28.47\% & 20.80\% & 14.82\% & 11.30\% \\
\bottomrule
\end{tabular}
\end{table}

\begin{table}[p]
\centering
\caption{Granular volatility thresholds for assets PEPEUSDT through XRPUSDT across experimental resolutions.}
\label{tab:coin_volatility_part2}
\footnotesize
\begin{tabular}{lcccccc}
\toprule
\textbf{Coin} & \textbf{Timeframe} & \textbf{$>1\%$} & \textbf{$>2\%$} & \textbf{$>3\%$} & \textbf{$>4\%$} & \textbf{$>5\%$} \\
\midrule
PEPEUSDT & 5m & 5.22\% & 0.52\% & 0.14\% & 0.05\% & 0.03\% \\
PEPEUSDT & 15m & 9.48\% & 1.84\% & 0.63\% & 0.28\% & 0.14\% \\
PEPEUSDT & 30m & 14.31\% & 3.95\% & 1.54\% & 0.72\% & 0.39\% \\
PEPEUSDT & 1h & 19.11\% & 7.02\% & 3.26\% & 1.77\% & 1.11\% \\
PEPEUSDT & 2h & 25.26\% & 11.63\% & 6.28\% & 3.64\% & 2.18\% \\
PEPEUSDT & 4h & 30.97\% & 17.96\% & 11.14\% & 7.07\% & 4.75\% \\
PEPEUSDT & 1d & 41.08\% & 34.35\% & 29.46\% & 24.70\% & 20.87\% \\
\midrule
SHIBUSDT & 5m & 1.36\% & 0.30\% & 0.11\% & 0.05\% & 0.02\% \\
SHIBUSDT & 15m & 3.75\% & 0.99\% & 0.43\% & 0.22\% & 0.12\% \\
SHIBUSDT & 30m & 6.71\% & 1.97\% & 0.93\% & 0.50\% & 0.29\% \\
SHIBUSDT & 1h & 10.88\% & 3.62\% & 1.70\% & 1.01\% & 0.64\% \\
SHIBUSDT & 2h & 16.34\% & 6.54\% & 3.28\% & 1.98\% & 1.36\% \\
SHIBUSDT & 4h & 22.25\% & 10.87\% & 5.90\% & 3.87\% & 2.70\% \\
SHIBUSDT & 1d & 37.25\% & 27.53\% & 19.16\% & 14.30\% & 10.59\% \\
\midrule
SOLUSDT & 5m & 1.49\% & 0.25\% & 0.07\% & 0.03\% & 0.02\% \\
SOLUSDT & 15m & 4.84\% & 1.10\% & 0.40\% & 0.16\% & 0.08\% \\
SOLUSDT & 30m & 8.98\% & 2.58\% & 1.04\% & 0.48\% & 0.23\% \\
SOLUSDT & 1h & 14.58\% & 5.16\% & 2.25\% & 1.10\% & 0.63\% \\
SOLUSDT & 2h & 21.18\% & 9.40\% & 4.77\% & 2.74\% & 1.62\% \\
SOLUSDT & 4h & 28.27\% & 15.44\% & 9.10\% & 5.66\% & 3.92\% \\
SOLUSDT & 1d & 40.99\% & 33.41\% & 26.97\% & 21.78\% & 17.45\% \\
\midrule
SUIUSDT & 5m & 1.04\% & 0.11\% & 0.02\% & 0.01\% & 0.01\% \\
SUIUSDT & 15m & 4.55\% & 0.71\% & 0.18\% & 0.05\% & 0.02\% \\
SUIUSDT & 30m & 9.10\% & 2.02\% & 0.63\% & 0.21\% & 0.09\% \\
SUIUSDT & 1h & 15.37\% & 4.78\% & 1.91\% & 0.78\% & 0.34\% \\
SUIUSDT & 2h & 21.73\% & 9.34\% & 4.15\% & 2.16\% & 1.04\% \\
SUIUSDT & 4h & 27.94\% & 15.25\% & 8.65\% & 4.98\% & 3.14\% \\
SUIUSDT & 1d & 38.60\% & 29.78\% & 24.37\% & 20.82\% & 16.86\% \\
\midrule
TONUSDT & 5m & 0.30\% & 0.05\% & 0.01\% & 0.01\% & 0.00\% \\
TONUSDT & 15m & 1.43\% & 0.23\% & 0.06\% & 0.03\% & 0.01\% \\
TONUSDT & 30m & 3.39\% & 0.53\% & 0.18\% & 0.07\% & 0.04\% \\
TONUSDT & 1h & 7.80\% & 1.25\% & 0.46\% & 0.22\% & 0.08\% \\
TONUSDT & 2h & 12.87\% & 3.20\% & 0.93\% & 0.48\% & 0.28\% \\
TONUSDT & 4h & 20.40\% & 6.63\% & 2.53\% & 1.24\% & 0.56\% \\
TONUSDT & 1d & 35.14\% & 23.99\% & 16.55\% & 11.15\% & 7.43\% \\
\midrule
TRXUSDT & 5m & 0.57\% & 0.08\% & 0.02\% & 0.01\% & 0.00\% \\
TRXUSDT & 15m & 2.10\% & 0.40\% & 0.12\% & 0.05\% & 0.02\% \\
TRXUSDT & 30m & 4.10\% & 1.01\% & 0.35\% & 0.15\% & 0.09\% \\
TRXUSDT & 1h & 7.22\% & 2.06\% & 0.91\% & 0.43\% & 0.22\% \\
TRXUSDT & 2h & 11.92\% & 4.24\% & 1.93\% & 1.03\% & 0.60\% \\
TRXUSDT & 4h & 17.73\% & 7.60\% & 3.94\% & 2.21\% & 1.32\% \\
TRXUSDT & 1d & 35.15\% & 23.64\% & 16.34\% & 11.55\% & 8.25\% \\
\midrule
XLMUSDT & 5m & 0.86\% & 0.14\% & 0.04\% & 0.02\% & 0.01\% \\
XLMUSDT & 15m & 2.92\% & 0.60\% & 0.21\% & 0.10\% & 0.05\% \\
XLMUSDT & 30m & 5.65\% & 1.38\% & 0.52\% & 0.26\% & 0.13\% \\
XLMUSDT & 1h & 9.74\% & 2.89\% & 1.25\% & 0.61\% & 0.35\% \\
XLMUSDT & 2h & 15.29\% & 5.49\% & 2.49\% & 1.37\% & 0.81\% \\
XLMUSDT & 4h & 21.71\% & 9.71\% & 4.97\% & 3.05\% & 1.87\% \\
XLMUSDT & 1d & 37.19\% & 26.44\% & 19.12\% & 14.12\% & 10.29\% \\
\midrule
XRPUSDT & 5m & 0.90\% & 0.15\% & 0.05\% & 0.02\% & 0.01\% \\
XRPUSDT & 15m & 2.71\% & 0.64\% & 0.23\% & 0.10\% & 0.06\% \\
XRPUSDT & 30m & 5.05\% & 1.38\% & 0.57\% & 0.28\% & 0.16\% \\
XRPUSDT & 1h & 8.63\% & 2.73\% & 1.19\% & 0.62\% & 0.37\% \\
XRPUSDT & 2h & 13.86\% & 4.95\% & 2.43\% & 1.35\% & 0.79\% \\
XRPUSDT & 4h & 19.87\% & 8.76\% & 4.67\% & 2.91\% & 1.93\% \\
XRPUSDT & 1d & 35.33\% & 24.96\% & 17.61\% & 12.85\% & 10.26\% \\
\bottomrule
\end{tabular}
\end{table}

\section{Experimental Details}
\label{sec:experimental_details}

To facilitate full reproducibility, this section details the concrete training configurations, hyperparameters, and optimization protocols employed across the four primary experimental tables of the manuscript. 

For the evaluation presented in \textbf{Table 1} (comparing models without normalization, standard RevIN, and our proposed TP-RevIN), all backbones are trained strictly using the standard Mean Squared Error (MSE) forecasting objective without the auxiliary physics constraint loss ($\lambda_{\text{phy}} = 0$). The input context window is fixed to $L = 480$ steps. A separate, independent experiment is conducted for each timeframe ($5$m, $30$m) and target forecast horizon ($H \in \{5, 15\}$). All configurations utilize a fixed numerical stabilization constant of $\epsilon = 10^{-5}$ in their normalization denominators. The Full results of Table 1 in the main paper are presented in Table \ref{tab:main}.

The comparison against alternative temporal normalization baselines in \textbf{Table 2} (FAN and SAN) is conducted using the $1$h timeframe with a lookback window of $L = 480$. All models in this evaluation are trained using the MSE objective only, without any auxiliary physics-informed losses. To ensure a fair and optimal implementation of each baseline, we adopt their recommended training schedules:
\begin{itemize}
    \item \textbf{FAN:} The forecasting backbone and the frequency adaptive projection module are optimized jointly for $3$ epochs per experiment.
    \item \textbf{SAN:} A decoupled, two-stage training strategy is adopted. First, the non-stationary slice projection module is optimized independently for $3$ epochs until convergence. This projection module is subsequently frozen, and the main forecasting backbone is trained for $1$ epoch.
    \item \textbf{TP-RevIN:} The model is trained for $1$ epoch using the scale-adaptive dynamic epsilon formulation ($\epsilon^{\text{dyn}} = 10^{-5}(\mu^2 + 10^{-12})$) to handle cross-asset heterogeneity.
\end{itemize}

For the numerical stabilization analysis in \textbf{Table 3} (comparing fixed versus dynamic epsilon), all configurations are evaluated on the $2$h timeframe with an input context length of $L = 480$. The models are trained strictly under the MSE objective to isolate the structural impact of denominator regularization on low-valued assets (e.g., SHIB) versus high-valued assets (e.g., BTC).

For the physical consistency evaluations in \textbf{Table 4}, all models are trained with a lookback window of $L = 480$ across different historical sampling intervals ($30$m and $1$h). Under the standard RevIN configuration, the auxiliary candlestick constraint loss is evaluated in the original physical coordinate space, which reintroduces absolute context scales. All of the evaluations are done under original space for fair comparison.

\begin{table*}[t]
\centering

\resizebox{\textwidth}{!}{%
\begin{tabular}{c c l | c c c | c c c | c c c}
\toprule

\multirow{2}{*}{\textbf{\shortstack{Time\\Frame}}} & \multirow{2}{*}{\textbf{Horizon}} & \multirow{2}{*}{\textbf{Norm Technique}} & 
\multicolumn{3}{c}{\textbf{Time-MoE}} & 
\multicolumn{3}{c}{\textbf{Timer-XL}} & 
\multicolumn{3}{c}{\textbf{Timer}} \\ 
\cmidrule(lr){4-6} \cmidrule(lr){7-9} \cmidrule(lr){10-12}

& & & 
\textbf{MSE} & \textbf{MAE} & \textbf{PHY} & 
\textbf{MSE} & \textbf{MAE} & \textbf{PHY} & 
\textbf{MSE} & \textbf{MAE} & \textbf{PHY} \\
\midrule

\multirow{15}{*}{\textbf{5m}} 
 & \multirow{5}{*}{5}  
 & w/o RevIN & 5.7e+8 & 6325.3 & 1.312 & 5.2e+8 & 6571.64 & 5.536 & 6.15e+8 & 6548.55 & 15.154 \\
 & & CI RevIN  & 9890.9 & 19.03 & 0.0040 & \textcolor{red}{4387.64} & \textcolor{red}{11.66} & 0.012 & 4722.96 & 12.37 & \textcolor{blue}{0.439} \\
 & & CD RevIN  & 9868.58 & 18.93 & 0.0076 & \textcolor{blue}{4485.7} & \textcolor{blue}{11.81} & 0.0050 & 5013.64 & 12.70 & 1.71 \\
 \rowcolor{LightGray} \cellcolor{white} & \cellcolor{white} & Two-Phase CI RevIN & \textcolor{red}{6610.0} & \textcolor{red}{15.25} & \textcolor{blue}{0.001} & 4584.75 & 12.34 & \textcolor{blue}{0.0005} & \textcolor{red}{4510.90} & \textcolor{red}{12.14} & \textcolor{red}{0.278} \\
 \rowcolor{LightGray} \cellcolor{white} & \cellcolor{white} & Two-Phase CD RevIN & \textcolor{blue}{7480.2} & \textcolor{blue}{16.31} & \textcolor{red}{0.00049} & 4811.0 & 12.72 & \textcolor{red}{0.0004} & \textcolor{blue}{4632.05} & \textcolor{blue}{12.30} & 1.09 \\
 \cmidrule{2-12}

 & \multirow{5}{*}{15} 
 & w/o RevIN & 5.7e+8 & 6328.50 & 1.827 & 5.5e+8 & 6572.12 & 5.532 & 6.15e+8 & 6548.99 & 9.523 \\
 & & CI RevIN  & 2.20e+4 & 28.63 & 0.0046 & 1.05e+4 & 18.54 & 0.0159 & 1.06e+4 & 18.69 & \textcolor{blue}{0.915} \\
 & & CD RevIN  & 2.23e+4 & 28.62 & \textcolor{blue}{0.0025} & 1.07e+4 & 18.71 & 0.0061 & 1.19e+4 & 19.73 & 6.285 \\
 \rowcolor{LightGray} \cellcolor{white} & \cellcolor{white} & Two-Phase CI RevIN & \textcolor{red}{1.53e+4} & \textcolor{red}{23.55} & 0.012 & \textcolor{blue}{1.02e+4} & \textcolor{blue}{18.39} & \textcolor{blue}{0.0003} & \textcolor{red}{9949.94} & \textcolor{red}{18.09} & \textcolor{red}{0.390} \\
 \rowcolor{LightGray} \cellcolor{white} & \cellcolor{white} & Two-Phase CD RevIN & \textcolor{blue}{1.95e+4} & \textcolor{blue}{26.83} & \textcolor{red}{0.0010} & \textcolor{red}{1.01e+4} & \textcolor{red}{18.32} & \textcolor{red}{0.0002} & \textcolor{blue}{1.04e+4} & \textcolor{blue}{18.54} & 4.36 \\
 \cmidrule{2-12}

 & \multirow{5}{*}{30} 
 & w/o RevIN & 5.76e+8 & 6327.02 & 2.189 & 5.9e+8 & 6748.77 & 2.32 & 6.15e+8 & 6547.06 & 9.84 \\
 & & CI RevIN  & 3.37e+4 & 34.82 & \textcolor{blue}{0.0016} & 1.74e+4 & 24.07 & 0.011 & 1.74e+4 & 23.94 & \textcolor{blue}{1.38} \\
 & & CD RevIN  & 3.36e+4 & 34.81 & 0.0024 & 1.73e+4 & 23.95 & 0.0071 & 1.99e+4 & 25.81 & 12.70 \\
 \rowcolor{LightGray} \cellcolor{white} & \cellcolor{white} & Two-Phase CI RevIN & \textcolor{red}{2.41e+4} & \textcolor{red}{29.36} & 0.015 & \textcolor{blue}{1.69e+4} & \textcolor{blue}{23.70} & \textcolor{blue}{0.0003} & \textcolor{red}{1.60e+4} & \textcolor{red}{23.10} & \textcolor{red}{0.555} \\
 \rowcolor{LightGray} \cellcolor{white} & \cellcolor{white} & Two-Phase CD RevIN & \textcolor{blue}{3.14e+4} & \textcolor{blue}{33.64} & \textcolor{red}{0.00072} & \textcolor{red}{1.65e+4} & \textcolor{red}{23.44} & \textcolor{red}{0.00003} & \textcolor{blue}{1.73e+4} & \textcolor{blue}{23.94} & 8.90 \\
\midrule\midrule

\multirow{15}{*}{\textbf{30m}} 
 & \multirow{5}{*}{5}  
 & w/o RevIN & 7.73e+8 & 8438.67 & 1.271 & 7.66e+8 & 8367.45 & 0.501 & 7.71e+8 & 8408.11 & 5.717 \\
 & & CI RevIN  & 6.47e+4 & 54.60 & 0.0685 & 3.35e+4 & 37.58 & 0.063 & 3.22e+4 & 36.13 & \textcolor{blue}{1.012} \\
 & & CD RevIN  & 6.33e+4 & 54.20 & 0.0307 & 3.38e+4 & 37.97 & 0.081 & 3.24e+4 & 36.39 & 3.957 \\
 \rowcolor{LightGray} \cellcolor{white} & \cellcolor{white} & Two-Phase CI RevIN & \textcolor{red}{4.67e+4} & \textcolor{red}{45.45} & \textcolor{blue}{0.011} & \textcolor{red}{3.11e+4} & \textcolor{red}{36.04} & \textcolor{blue}{0.003} & \textcolor{blue}{3.03e+4} & \textcolor{blue}{34.93} & \textcolor{red}{0.678} \\
 \rowcolor{LightGray} \cellcolor{white} & \cellcolor{white} & Two-Phase CD RevIN & \textcolor{blue}{5.41e+4} & \textcolor{blue}{49.12} & \textcolor{red}{0.0076} & \textcolor{blue}{3.11e+4} & \textcolor{blue}{36.26} & \textcolor{red}{0.00074} & \textcolor{red}{1.4e+4} & \textcolor{red}{18.54} & 4.365 \\
 \cmidrule{2-12}

 & \multirow{5}{*}{15} 
 & w/o RevIN & 7.73e+8 & 8442.13 & 1.553 & 7.68e+8 & 8383.54 & 0.606 & 7.71e+8 & 8412.84 & 5.542 \\
 & & CI RevIN  & 1.56e+5 & 86.77 & 0.0701 & 8.29e+4 & 59.11 & 0.228 & 7.66e+4 & 57.00 & \textcolor{blue}{3.238} \\
 & & CD RevIN  & 1.55e+5 & 87.05 & \textcolor{blue}{0.0443} & 8.49e+4 & 60.007 & 0.163 & 8.51e+4 & 59.75 & 20.67 \\
 \rowcolor{LightGray} \cellcolor{white} & \cellcolor{white} & Two-Phase CI RevIN & \textcolor{red}{1.14e+5} & \textcolor{red}{72.48} & 0.12 & \textcolor{blue}{7.34e+4} & \textcolor{blue}{55.10} & \textcolor{blue}{0.026} & \textcolor{red}{6.99e+4} & \textcolor{red}{53.38} & \textcolor{red}{1.123} \\
 \rowcolor{LightGray} \cellcolor{white} & \cellcolor{white} & Two-Phase CD RevIN & \textcolor{blue}{1.37e+5} & \textcolor{blue}{80.14} & \textcolor{red}{0.0078} & \textcolor{red}{7.24e+4} & \textcolor{red}{54.71} & \textcolor{red}{0.00052} & \textcolor{blue}{7.36e+4} & \textcolor{blue}{55.11} & 15.002 \\
 \cmidrule{2-12}

 & \multirow{5}{*}{30} 
 & w/o RevIN & 7.74e+8 & 8449.87 & 1.643 & 7.69e+8 & 8392.46 & 0.522 & 7.72e+8 & 8424.90 & 6.923 \\
 & & CI RevIN  & 2.29e+5 & 106.78 & \textcolor{blue}{0.129} & 1.39e+5 & 78.05 & 0.082 & 1.34e+5 & 75.79 & \textcolor{blue}{6.036} \\
 & & CD RevIN  & 2.34e+5 & 107.85 & 0.214 & 1.43e+5 & 78.81 & 0.153 & 1.48e+5 & 80.23 & 39.168 \\
 \rowcolor{LightGray} \cellcolor{white} & \cellcolor{white} & Two-Phase CI RevIN & \textcolor{red}{1.79e+5} & \textcolor{red}{91.87} & 0.562 & \textcolor{red}{1.21e+5} & \textcolor{red}{71.71} & \textcolor{blue}{0.001} & \textcolor{red}{1.16e+5} & \textcolor{red}{70.08} & \textcolor{red}{1.928} \\
 \rowcolor{LightGray} \cellcolor{white} & \cellcolor{white} & Two-Phase CD RevIN & \textcolor{blue}{2.19e+5} & \textcolor{blue}{101.75} & \textcolor{red}{0.0180} & \textcolor{blue}{1.22e+5} & \textcolor{blue}{71.92} & \textcolor{red}{0.001} & \textcolor{blue}{1.26e+5} & \textcolor{blue}{73.74} & 34.545 \\
\midrule\midrule

\multirow{15}{*}{\textbf{1h}} 
 & \multirow{5}{*}{5}  
 & w/o RevIN & 7.8e+8 & 8649.4 & 1.36 & 7.80e+8 & 8610.74 & 0.136 & 7.82e+8 & 8643.78 & 2.214 \\
 & & CI RevIN  & 1.27e+5 & 79.73 & 0.139 & 6.84e+4 & 56.16 & 0.146 & 6.46e+4 & 52.80 & \textcolor{blue}{1.571} \\
 & & CD RevIN  & 1.25e+5 & 79.68 & 0.084 & 6.90e+4 & 56.82 & 0.122 & 6.57e+4 & 52.94 & 4.97 \\
 \rowcolor{LightGray} \cellcolor{white} & \cellcolor{white} & Two-Phase CI RevIN & \textcolor{red}{9.38e+4} & \textcolor{red}{65.91} & \textcolor{red}{0.02} & \textcolor{blue}{6.48e+4} & \textcolor{blue}{54.44} & \textcolor{blue}{0.008} & \textcolor{red}{6.16e+4} & \textcolor{red}{51.25} & \textcolor{red}{1.275} \\
 \rowcolor{LightGray} \cellcolor{white} & \cellcolor{white} & Two-Phase CD RevIN & \textcolor{blue}{1.08e+5} & \textcolor{blue}{71.86} & \textcolor{blue}{0.042} & \textcolor{red}{6.16e+4} & \textcolor{red}{52.55} & \textcolor{red}{0.0032} & \textcolor{blue}{6.34e+4} & \textcolor{blue}{51.83} & 5.247 \\
 \cmidrule{2-12}

 & \multirow{5}{*}{15} 
 & w/o RevIN & 7.85e+8 & 8680.34 & 1.532 & 7.81e+8 & 8629.43 & 0.302 & 7.83e+8 & 8646.39 & \textcolor{red}{1.833} \\
 & & CI RevIN  & 3.05e+5 & 128.34 & 0.296 & 1.73e+5 & 90.42 & 0.455 & 1.62e+5 & 86.06 & 6.20 \\
 & & CD RevIN  & 3.05e+5 & 127.82 & \textcolor{blue}{0.234} & 1.77e+5 & 90.61 & 0.520 & 1.71e+5 & 88.88 & 29.87 \\
 \rowcolor{LightGray} \cellcolor{white} & \cellcolor{white} & Two-Phase CI RevIN & \textcolor{red}{2.31e+5} & \textcolor{red}{108.28} & 0.571 & \textcolor{blue}{1.48e+5} & \textcolor{blue}{82.80} & \textcolor{blue}{0.055} & \textcolor{red}{1.39e+5} & \textcolor{red}{78.92} & \textcolor{blue}{2.365} \\
 \rowcolor{LightGray} \cellcolor{white} & \cellcolor{white} & Two-Phase CD RevIN & \textcolor{blue}{2.75e+5} & \textcolor{blue}{119.51} & \textcolor{red}{0.063} & \textcolor{red}{1.47e+5} & \textcolor{red}{82.38} & \textcolor{red}{0.0042} & \textcolor{blue}{1.48e+5} & \textcolor{blue}{82.25} & 23.58 \\
 \cmidrule{2-12}

 & \multirow{5}{*}{30} 
 & w/o RevIN & 7.85e+8 & 8691.64 & 1.728 & 7.82e+8 & 8647.22 & 0.290 & 7.84e+8 & 8670.49 & \textcolor{red}{2.358} \\
 & & CI RevIN  & 4.21e+5 & 151.30 & \textcolor{blue}{0.509} & 2.80e+5 & 116.25 & 0.299 & \textcolor{blue}{2.58e+5} & \textcolor{blue}{110.42} & \textcolor{blue}{10.22} \\
 & & CD RevIN  & 4.32e+5 & 155.56 & 0.743 & 2.79e+5 & 116.94 & 0.476 & 2.95e+5 & 118.23 & 57.194 \\
 \rowcolor{LightGray} \cellcolor{white} & \cellcolor{white} & Two-Phase CI RevIN & \textcolor{red}{3.26e+5} & \textcolor{red}{130.22} & 2.123 & \textcolor{blue}{2.35e+5} & \textcolor{blue}{105.72} & \textcolor{blue}{0.017} & - & - & - \\
 \rowcolor{LightGray} \cellcolor{white} & \cellcolor{white} & Two-Phase CD RevIN & \textcolor{blue}{3.77e+5} & \textcolor{blue}{141.53} & \textcolor{red}{0.109} & \textcolor{red}{2.33e+5} & \textcolor{red}{105.22} & \textcolor{red}{0.0039} & \textcolor{red}{2.40e+5} & \textcolor{red}{107.54} & 52.350 \\

\bottomrule
\end{tabular}%
}
\caption{Comparison of Models with different Normalization Techniques. Lower numbers indicate lower errors and better performances.\\ Best metrics in each category are colored with \textcolor{red}{Red} and Second best metrics are colored with \textcolor{blue}{Blue}.}
\label{tab:main}
\end{table*}


\end{document}